%% file: main.tex
\pdfoutput=1

\input{arxiv-style}

\usepackage{graphicx}
\usepackage{multirow}
\usepackage{amsmath,amssymb,amsfonts}
\usepackage{amsthm}
\usepackage{mathtools}
\usepackage{mathrsfs}
\usepackage{bm}
\usepackage[title]{appendix}
\usepackage{xcolor}
\usepackage{textcomp}
\usepackage{manyfoot}
\usepackage{booktabs}
\usepackage{algorithm}
\usepackage{algorithmicx}
\usepackage{algpseudocode}
\usepackage{listings}
\usepackage{tikz}
\usepackage{pgfplots}
\pgfplotsset{compat=1.15}
\usepgfplotslibrary{groupplots}
\usetikzlibrary{shapes.geometric, arrows.meta, positioning, calc, patterns,
                decorations.pathreplacing, backgrounds, fit, matrix}

\definecolor{accent}{RGB}{0,62,116}
\definecolor{rulegrey}{RGB}{150,150,150}
\definecolor{textgrey}{RGB}{90,90,90}
\definecolor{serverblue}{RGB}{66,133,244}
\definecolor{clientgreen}{RGB}{52,168,83}
\definecolor{updateorange}{RGB}{251,188,5}
\definecolor{heterored}{RGB}{234,67,53}
\definecolor{lightblue}{RGB}{232,240,254}
\definecolor{lightgreen}{RGB}{232,250,234}
\definecolor{lightorange}{RGB}{255,248,232}

\theoremstyle{plain}
\newtheorem{theorem}{Theorem}
\newtheorem{proposition}[theorem]{Proposition}

\theoremstyle{remark}

\theoremstyle{definition}
\newtheorem{definition}[theorem]{Definition}

\newcommand{\E}{\mathbb{E}}
\newcommand{\R}{\mathbb{R}}

\newcommand{\cO}{\mathcal{O}}
\newcommand{\cD}{\mathcal{D}}
\newcommand{\cS}{\mathcal{S}}

\newcommand{\w}{\mathbf{w}}
\newcommand{\m}{\mathbf{m}}

\newcommand{\g}{\mathbf{g}}

\begin{document}

\title[Capacity Confounds and Coverage Guarantees in Adaptive Sub-model Federated Learning]{Capacity Confounds and Coverage Guarantees in Adaptive Sub-model Federated Learning}

\author*[1]{\fnm{Alireza} \sur{Moayedikia}}\email{amoayedikia@swin.edu.au}

\author[2]{\fnm{Alicia} \sur{Troncoso Lora}}\email{atrolor@upo.es}

\affil*[1]{\orgdiv{Department of Business Technology and Entrepreneurship, School of Business Law and Entrepreneurship}, \orgname{Swinburne University of Technology}, \orgaddress{\street{John Street}, \city{Hawthorn}, \postcode{3122}, \state{VIC}, \country{Australia}}}

\affil[2]{\orgdiv{Data Science \& Big Data Lab}, \orgname{Universidad Pablo de Olavide}, \orgaddress{\city{Seville}, \country{Spain}}}

\abstract{Sub-model federated learning allows clients with limited memory and compute to join training by optimising width-reduced slices of a shared global network. Established schemes such as HeteroFL and FjORD size these slices from device resources alone, and a growing line of adaptive methods proposes to also enlarge the slices of clients whose data deviate most from the population, inferring that deviation from the model updates the server already receives. This paper puts that design pattern to a controlled test. We construct HAS-FL, a representative heterogeneity-aware allocation framework that isolates this inference step, and evaluate it on CIFAR-10, EMNIST and the naturally partitioned Shakespeare corpus. We find that the inference step fails: divergence between local and global updates tracks the width of each client's sub-model, with a strong negative correlation in every image run, and carries little or no information about ground-truth label skew once width is held fixed. We also identify a separate failure of capped allocation: weights that no client trains remain at their random initial values yet still enter every forward pass, driving accuracy close to chance on the 62-class EMNIST task; guaranteeing that every parameter is trained by at least one client removes this failure. Finally, a random allocation with the same average budget matches the adaptive policy on the image tasks and surpasses it on text. The gains attributed to adaptive allocation therefore come from capacity budgeting and parameter coverage rather than heterogeneity estimation, and we recommend capacity-stratified validation and matched-budget random baselines as standard practice.}


\maketitle

\section{Introduction}
\label{sec:introduction}

Federated learning (FL) has emerged as a promising paradigm for privacy-preserving distributed machine learning, enabling multiple clients to collaboratively train a shared model without exchanging raw data~\citep{mcmahan2017communication,kairouz2021advances}.
This approach addresses critical privacy concerns in sensitive domains such as healthcare~\citep{dayan2021federated,rieke2020future}, mobile computing~\citep{hard2019federated}, and Internet of Things (IoT) applications~\citep{nguyen2021federated,khan2021federated}.
However, the practical deployment of FL systems faces two fundamental challenges that significantly hinder their effectiveness: statistical heterogeneity arising from non-independent and identically distributed (non-IID) data across clients, and system heterogeneity stemming from diverse computational capabilities among participating devices~\citep{li2020federated,ye2023heterogeneous}.

Statistical heterogeneity manifests when local data distributions differ substantially across clients, a phenomenon commonly referred to as client drift~\citep{karimireddy2020scaffold,zhu2021noniid}.
This heterogeneity can take multiple forms, including covariate shift where feature distributions vary while label distributions remain consistent, concept drift where the relationship between features and labels changes across clients, and quantity skew where data volumes differ dramatically~\citep{li2022federated,zhu2021noniid}; recent federated studies routinely simulate label, feature, and quantity skew in combination to reflect this variety~\citep{ducange2026fedshap}.
The consequences are severe: gradient updates computed on heterogeneous local data may point in conflicting directions, degrading the global model's convergence rate and final accuracy.
Furthermore, features learned by shared model components must simultaneously adapt to heterogeneous local distributions while serving the global collaborative objective, creating an inherent optimization conflict that complicates the training process~\citep{wang2026fedrda}.

System heterogeneity presents an orthogonal but equally critical challenge. In realistic FL deployments, participating devices exhibit significant disparities in computational power, memory capacity, and communication bandwidth~\citep{diao2021heterofl,caldas2018leaf}.
Resource-constrained clients such as mobile phones, IoT sensors, and edge devices may be unable to train or even store full-sized neural network models, leading to their exclusion from the federation.
This exclusion is not merely inefficient---it fundamentally biases the learned model toward data distributions present only on more capable devices, undermining the core premise of collaborative learning.
Recent surveys emphasize that addressing these disparities in client computational capabilities alongside non-IID local data distributions is essential for both personalization of local models and generalization of the global model~\citep{ye2023heterogeneous,kairouz2021advances}.

Existing approaches have addressed these challenges largely in isolation. For statistical heterogeneity, methods such as FedProx~\citep{li2020fedprox} add proximal regularization to limit local model drift, while SCAFFOLD~\citep{karimireddy2020scaffold} employs control variates to correct gradient bias.
More recent work explores representation decoupling strategies that explicitly separate global-shared features from client-specific representations to mitigate interference between collaborative and personalized learning objectives~\citep{wang2026fedrda,collins2021exploiting}.
For system heterogeneity, sub-model training approaches such as HeteroFL~\citep{diao2021heterofl} and FjORD~\citep{horvath2021fjord} allow resource-constrained clients to train smaller subsets of the global model.
However, a known limitation of static sub-model partitioning is that parameters outside the smaller sub-models are updated only by the clients whose sub-models contain them, so parts of the global model are trained unevenly~\citep{alam2022fedrolex}.
Despite this progress, the two lines of work have largely developed separately, and it has not been tested whether a sub-model server can recover client data heterogeneity from the updates it observes.
Uniform sub-model allocation ignores the varying degrees of data heterogeneity across clients, while heterogeneity-aware optimization methods assume all clients can train full models.

This paper asks a question prior to any such design: can client data heterogeneity be estimated at all from the signals a sub-model server observes? To answer it under controlled conditions, we construct Heterogeneity-Aware Adaptive Sub-model Federated Learning (HAS-FL), a representative adaptive capacity-allocation framework that extends the HeteroFL--FjORD line of sub-model training with the heterogeneity-inference step that recent adaptive proposals share.
The motivating hypothesis is that clients with highly non-IID data require larger model capacity to capture their distinctive local patterns, while clients with data more representative of the global distribution can effectively train smaller sub-models.
This intuition aligns with recent findings that adaptive mechanisms outperform static allocation strategies in heterogeneous environments~\citep{wang2026fedrda,lee2026flexchill}.
HAS-FL operationalizes the hypothesis through normalized gradient divergence---the squared norm of the difference between local and global updates, normalized by the magnitude of the global update---as an observable proxy for data heterogeneity that can be computed without additional privacy cost.
Our analysis, however, shows that this is far harder than it appears: the observable signal is dominated by capacity itself.

Our skepticism is grounded in a recurring pattern: careful, budget-matched evaluation has repeatedly levelled or overturned the claimed advantages of adaptive machinery.
In centralized learning, well-tuned SGD matches or exceeds adaptive optimizers~\citep{wilson2017marginal}, and a large-scale neutral comparison of GAN variants found no consistent winner once tuning budgets were equalized~\citep{lucic2018gans}.
In federated learning specifically, heterogeneity-aware re-evaluation on large-scale smartphone data erodes the reported gains of state-of-the-art aggregation and compression schemes~\citep{yang2021characterizing}, and an empirical study spanning some 1,500 configurations shows device and behavioral heterogeneity silently driving both model quality and fairness~\citep{abdelmoniem2023comprehensive}.
This paper applies the same discipline to adaptive sub-model allocation.

The contributions of this work are threefold.
First, we report a systematic analysis---validated against ground-truth label-distribution divergence on reproducible partitions---showing that update-divergence estimates of client heterogeneity are dominated by capacity heterogeneity: small-capacity clients diverge because their sub-models represent the task differently, not because their data is more skewed. The confound appears on both image benchmarks and on every seed, and to our knowledge has not been documented before.
Second, we identify and analyze a failure mode of capped sub-model allocation---parameters left uncovered by every client remain at their initialization and progressively degrade the global model---and introduce a coverage guarantee, together with coordinate-wise aggregation, that eliminates it; the same mechanism explains the previously unexplained collapse of uniform allocation.
Third, through a unified multi-seed evaluation on CIFAR-10, EMNIST, and a naturally partitioned Shakespeare benchmark, we show that a matched-budget random control matches heterogeneity-aware allocation to within seed-to-seed variation on the image benchmarks, that the adaptive policy is the weakest of the allocation strategies on the text benchmark, and that sub-model training as a whole gives up a substantial accuracy margin against full-model training in exchange for its reduced cost. The benefits that adaptive allocation does deliver come from capacity budgeting and parameter coverage rather than from heterogeneity awareness.

An implementation of every method and experiment reported here, together with the configuration files needed to reproduce each table and figure, is publicly available at \url{https://github.com/amoayedikia/hasfl}.

\section{Related Work}
\label{sec:related_work}

Federated learning research has produced a rich landscape of methods addressing the challenges of distributed, privacy-preserving machine learning.
We organize the most relevant prior work into four categories: federated optimization under heterogeneity, system-heterogeneous federated learning, personalized federated learning, and communication-efficient approaches.

\subsection{Federated Optimization Under Statistical Heterogeneity}

The foundational FedAvg algorithm~\citep{mcmahan2017communication} established the basic framework of local stochastic gradient descent (SGD) followed by model averaging, demonstrating empirical success in independent and identically distributed (IID) settings but exhibiting degraded performance under statistical heterogeneity.
Subsequent work has sought to understand and mitigate this degradation through various mechanisms.
FedProx~\citep{li2020fedprox} addresses client drift by adding a proximal term to the local objective, constraining local updates to remain close to the global model.
While effective in moderately heterogeneous settings, this approach uniformly penalizes all clients regardless of their actual divergence from the global distribution.
SCAFFOLD~\citep{karimireddy2020scaffold} takes a different approach by introducing control variates that estimate and correct for the drift between local and global gradients.
This method achieves variance reduction and improved convergence rates but requires additional communication of control variates and assumes clients can maintain state across rounds.
Recent advances have explored more sophisticated mechanisms for handling heterogeneity.
FedRDA~\citep{wang2026fedrda} proposes representation decoupling combined with divergence-aware aggregation, explicitly separating global-shared and client-specific feature representations while dynamically adjusting client contributions based on semantic divergence measured through task vectors.
This work demonstrates that static weighting mechanisms based on sample size cannot adequately address scenarios where critical task information may be diluted across highly heterogeneous clients.
Similarly, FLex\&Chill~\citep{lee2026flexchill} introduces temperature scaling during local training, where lower temperatures amplify gradient magnitudes and accelerate convergence under non-IID conditions.
Their theoretical analysis shows that gradient magnitude scales inversely with temperature, providing an optimization technique complementary to capacity-based approaches.
The challenge of non-stationary data in federated settings has also received attention.
\citet{zhu2021noniid} survey the forms that non-IID data takes in federated learning and analyze their influence on model training.
Their analysis reveals that existing methods often conflate these distinct sources of heterogeneity, potentially applying inappropriate corrections.
The method we study differs from these approaches by acting on capacity allocation rather than on gradient correction or aggregation weighting.
Rather than uniformly constraining or correcting all clients, HAS-FL sets how much model capacity each client trains from an observable heterogeneity estimate; whether that estimate carries the information the design needs is the question of Section~\ref{subsec:rq2}.

\subsection{System-Heterogeneous Federated Learning}

System heterogeneity in federated learning arises from the diverse computational capabilities of participating devices.
Early approaches addressed this through client selection strategies that exclude slower devices, but this introduces bias toward data present only on more capable clients.
HeteroFL~\citep{diao2021heterofl} pioneered sub-model training by allowing clients to train width-reduced versions of the global model according to their computational capacity.
The method aggregates sub-models through channel-wise averaging, enabling heterogeneous participation while maintaining a single global model.
FjORD~\citep{horvath2021fjord} extended this approach using ordered dropout, training nested sub-models that share parameters hierarchically.
Both methods demonstrate that partial model training can enable broader participation without catastrophic accuracy loss.
However, these approaches employ fixed or random sub-model allocation strategies that do not account for data heterogeneity.
FedRolex~\citep{alam2022fedrolex} identifies a limitation of static sub-model extraction: parameters outside the smaller sub-models are trained only by the clients large enough to hold them, so the global model is trained unevenly, which motivates its rolling extraction scheme.
None of these methods sizes sub-models by client data heterogeneity, which is the strategy we study: it would give clients with more distinctive data larger sub-models, and Section~\ref{subsec:rq2} examines whether the required estimate is obtainable in the first place.
PFedGate~\citep{chen2023efficient} introduces learnable gating mechanisms for structured block sparsity, allowing personalized sparse architectures that adapt to local data characteristics.
This architectural personalization approach is complementary to capacity-based allocation.
The method we analyze, HAS-FL, extends these schemes by coupling sub-model allocation to heterogeneity estimation---the natural next step suggested in this literature, and the one this paper puts to the test.
Rather than treating system and data heterogeneity as separate concerns, the method uses gradient divergence as a signal to inform capacity decisions, allocating larger sub-models to clients whose data appears to deviate more from the global distribution---the premise Section~\ref{subsec:rq2} tests directly.

A parallel systems literature treats device heterogeneity as a scheduling and selection problem while keeping every client's model full-size.
Oort selects participants by a utility score that deliberately blends statistical utility (local training loss) with system utility (measured latency)~\citep{lai2021oort}, and PyramidFL refines this joint profiling at finer granularity~\citep{li2022pyramidfl}; AutoFL schedules participants and per-device execution settings for energy efficiency~\citep{kim2021autofl}, and FedBalancer selects informative samples under device-speed-aware deadlines~\citep{shin2022fedbalancer}.
Because these systems train equal-size models, their published results are untouched by sub-model effects---but their utility signals already fuse data and capacity by design.
A second group sizes the sub-model itself from training-derived signals: FedMP chooses per-worker pruning ratios with a capability-driven bandit~\citep{jiang2022fedmp}, Hermes prunes channels by local-data importance~\citep{li2021hermes}, FIARSE composes each client's sub-model from parameter-magnitude importance up to the device budget~\citep{wu2024fiarse}, and FedLPS explicitly learns per-client sparse ratios from the superimposed effect of device capability and non-IID data~\citep{xue2025fedlps}.
By contrast, FedRolex rotates which slice trains purely by round index---a capacity-only rule involving no data-derived signal~\citep{alam2022fedrolex}.
Our analysis (Section~\ref{subsec:rq2}) shows why this distinction matters: once sub-model width varies with device capacity, the update-based signal we study conflates data and capacity, and we could not recover its data component.

An alternative route to resource-constrained participation is split learning, in which the network is partitioned at a cut layer and each client trains only the layers before the cut while the server computes the remainder~\citep{vepakomma2018split}; SplitFed combines this partitioning with federated aggregation to remove the sequential client bottleneck~\citep{thapa2022splitfed}. Split approaches shift most computation to the server, at the price of transmitting activations and gradients for every training sample and involving the server in every forward and backward pass; sub-model training instead keeps clients self-contained and communicates once per round, at the price of each client holding a width-reduced copy of the model. The two paradigms therefore suit different deployment regimes and are complementary rather than competing. The confound we study arises specifically when a server must infer client properties from heterogeneously sized local updates---a setting split learning does not create, since every client there trains client-side layers of identical size.

\subsection{Personalized Federated Learning}

Personalized federated learning (pFL) produces models tailored to individual clients rather than a single global model.
Per-FedAvg~\citep{fallah2020personalized} treats the global model as an initialization for rapid local adaptation, Ditto~\citep{li2021ditto} maintains separate global and local models joined by regularization, and FedRep~\citep{collins2021exploiting} partitions the model into shared representation layers and personalized classifier heads; prototype-based methods such as FedProto~\citep{tan2022fedproto} share class prototypes rather than model parameters.
These methods require every client to train a full-sized model, so they do not address capacity constraints.
Heterogeneity-aware allocation would, in principle, provide implicit personalization by giving clients with distinctive data larger sub-models; our results indicate that the premise this rests on does not hold under capacity heterogeneity.

\subsection{Communication Efficiency and One-Shot Federated Learning}

Communication cost is a second practical constraint on FL deployment.
Gradient compression reduces the size of transmitted updates through quantization~\citep{alistarh2017qsgd}, sparsification~\citep{aji2017sparse}, or low-rank approximation~\citep{vogels2019powersgd}; these methods are orthogonal to sub-model training and can be combined with any sub-model scheme.
One-shot federated learning goes further and completes training in a single communication round, giving up iterative refinement in exchange~\citep{amato2026oneshot,zhao2026feddistr}. A related line has clients transmit compact summaries of their local data rather than model updates: federated discretization, for example, has each client fit Gaussian mixture models to its continuous attributes and send only their parameters, from which the server derives a global policy in one round with accuracy comparable to centralized processing~\citep{flores2026feddisc}.
Sub-model training occupies a middle ground: by enabling resource-constrained clients to participate at reduced width, it lowers per-round computation and communication while retaining multi-round refinement; whether it can additionally be made heterogeneity-aware is the question we take up.

\section{The Method Under Study}
\label{sec:method}

This section specifies HAS-FL, the adaptive sub-model method whose premises this paper tests. We present it in full because the analysis in Section~\ref{sec:experiments} depends on the exact form of its heterogeneity estimator, its allocation rule, and its aggregation scheme; the design follows the natural extension of resource-only sub-model training toward which recent methods that size sub-models from training-derived signals point~\citep{li2021hermes,wu2024fiarse,xue2025fedlps}.
We begin by formalizing the problem setting in Section~\ref{subsec:problem}, then define the heterogeneity measure it relies on in Section~\ref{subsec:heterogeneity_measure}, describe the complete algorithm with its adaptive allocation mechanism in Section~\ref{subsec:allocation}, and finally analyze communication and computation complexity in Section~\ref{subsec:complexity}.

\subsection{Problem Formulation}
\label{subsec:problem}

Consider a federated learning system with $N$ clients collaboratively training a shared model.
Each client $i \in [N] := \{1, 2, \ldots, N\}$ possesses a local dataset $\cD_i$ drawn from a potentially distinct distribution $P_i$, and the global objective is to minimize the aggregate loss:
\begin{equation}
\label{eq:global_objective}
\begin{aligned}
&\min_{\w \in \R^d} F(\w) := \frac{1}{N} \sum_{i=1}^{N} f_i(\w),\\
&\text{where} \quad f_i(\w) := \E_{\xi \sim P_i}[\ell(\w; \xi)],
\end{aligned}
\end{equation}
with $\ell(\w; \xi)$ denoting the loss function evaluated at parameters $\w$ on data sample $\xi$, and $d$ representing the model dimensionality.
In practical federated deployments, clients exhibit heterogeneity along two orthogonal dimensions.
First, \emph{resource heterogeneity} manifests as varying computational and memory capacities across devices.
We formalize this through a capacity constraint $p_i^{\max} \in (0, 1]$ for each client $i$, representing the largest \emph{width fraction} that client $i$ can store and update, that is, the largest fraction of each layer's channels. Second, \emph{data heterogeneity} arises from non-identical local distributions $P_i \neq P_j$ for $i \neq j$, causing local objectives $f_i$ to diverge from the global objective $F$.

The sub-model training paradigm addresses resource heterogeneity by having each client $i$ train only a subset of parameters determined by a binary mask $\m_i \in \{0, 1\}^d$.
Following the nested (ordered) sub-model construction of HeteroFL~\citep{diao2021heterofl} and FjORD~\citep{horvath2021fjord}, we instantiate $\m_i$ as a \emph{structured width mask}: a client with capacity $p_i$ trains the leading $\lceil p_i C_\ell \rceil$ channels of every layer $\ell$, together with the corresponding input connections, so that sub-models of increasing capacity are nested. Throughout, capacity $p_i$ denotes this width fraction. The fraction of parameters that client $i$ actually holds and updates is $\rho_i := \|\m_i\|_0/d$, and because slicing shrinks both the input and the output channels of a layer, $\rho_i \approx p_i^2$ (Figure~\ref{fig:cost_model}).
The masked local update becomes:
\begin{equation}
\label{eq:masked_update}
\w_i^{(t+1)} = \w^{(t)} - \eta \cdot \m_i^{(t)} \odot \nabla f_i(\m_i^{(t)} \odot \w^{(t)}; \xi_i^{(t)}),
\end{equation}
where $\odot$ denotes element-wise multiplication and $\eta$ is the learning rate.

The design question behind HAS-FL is: \emph{How should the sub-model allocation $\{p_i\}_{i=1}^N$ be determined when clients exhibit both resource constraints and heterogeneous data distributions?} Existing approaches either ignore data heterogeneity when allocating sub-models or treat resource and data heterogeneity independently.
The premise under test is that these two factors cannot be treated independently: the clients that would benefit most from additional capacity are not necessarily those that possess it.
The allocation rule we study, formalized in Equation~\eqref{eq:allocation_rule}, distributes capacity according to heterogeneity estimates while respecting the constraints $p_i^{\max}$.

\subsection{Heterogeneity Measure}
\label{subsec:heterogeneity_measure}

Adaptive sub-model allocation requires a quantitative measure of each client's data heterogeneity that the server can compute from what it observes.
The measure under study is the \emph{normalized gradient divergence}.

\begin{definition}[Normalized Gradient Divergence]
\label{def:heterogeneity}
For client $i$ at round $t$, the normalized gradient divergence is defined as:
\begin{equation}
\label{eq:heterogeneity}
H_i^{(t)} := \frac{\|\nabla f_i(\w^{(t)}) - \nabla F(\w^{(t)})\|^2}{\|\nabla F(\w^{(t)})\|^2 + \epsilon},
\end{equation}
where $\epsilon > 0$ is a small constant ensuring numerical stability.
\end{definition}

This measure captures how much client $i$'s local gradient deviates from the global gradient direction, normalized by the global gradient magnitude.
Clients with $H_i^{(t)} \approx 0$ have data distributions aligned with the population, while large $H_i^{(t)}$ indicates significant distributional shift.
Gradient dissimilarity of this kind appears in convergence bounds for federated optimization~\citep{li2020convergence,karimireddy2020scaffold}, which is what makes it a natural candidate for guiding allocation.
In practice, the server cannot directly compute $H_i^{(t)}$ as it lacks access to local gradients.
However, it can be estimated from the observable client updates.
After each communication round, the server receives updates $\Delta_i^{(t)} := \w_i^{(t+1)} - \w^{(t)}$ from participating clients.
The estimator under study is:
\begin{equation}
\label{eq:heterogeneity_estimate}
\hat{H}_i^{(t)} := \frac{\|\m_i^{(t)} \odot (\Delta_i^{(t)} - \bar{\Delta}^{(t)})\|^2}{\|\m_i^{(t)} \odot \bar{\Delta}^{(t)}\|^2 + \epsilon},
\end{equation}
where $\bar{\Delta}^{(t)}$ denotes the coordinate-wise aggregated update computed by the server (Algorithm~\ref{alg:hasfl}, line 18) over the set $\cS^{(t)}$ of clients participating in round $t$.
This estimator leverages the relationship between gradients and parameter updates under gradient descent and serves as a practical server-side proxy for $H_i^{(t)}$; its absolute scale depends on the learning rate and the number of local steps, but only the \emph{relative} ordering of clients enters the allocation rule. Restricting the computation to the coordinates in $\m_i^{(t)}$ is essential: without this restriction, every coordinate a low-capacity client did not train would contribute the full magnitude of $\bar{\Delta}^{(t)}$ to its divergence, systematically inflating the heterogeneity estimates of low-capacity clients and creating a spurious feedback loop between capacity and estimated heterogeneity.
Whether \emph{any} update-based statistic can recover data heterogeneity under capacity heterogeneity is an empirical question; Section~\ref{subsec:rq2} answers it with ground-truth validation.
To reduce variance in the heterogeneity estimates, we employ exponential moving average smoothing:
\begin{equation}
\label{eq:smoothed_heterogeneity}
\tilde{H}_i^{(t)} := \beta \cdot \tilde{H}_i^{(t-1)} + (1 - \beta) \cdot \hat{H}_i^{(t)},
\end{equation}

where $\beta \in [0, 1)$ is the smoothing parameter controlling the trade-off between responsiveness and stability, and $\tilde{H}_i^{(0)} = 1$ for all clients as the initial estimate before any observations are available.
Larger $\beta$ values produce more stable estimates but slower adaptation to changing heterogeneity patterns.
Because the absolute scale of the estimates drifts over training (Figure~\ref{fig:rq2_dynamics}), all estimates are rescaled to unit mean every $T_{\text{norm}}$ rounds. The rescaling preserves the relative ordering of clients and leaves the allocation unchanged at the round in which it is applied, since the rule in Equation~\eqref{eq:allocation_rule} depends only on the ratio $\tilde{H}_i/\bar{H}$; its only lasting effect is on how past and new observations are weighted in the moving average.

\subsection{Adaptive Sub-model Allocation}
\label{subsec:allocation}

The central mechanism of HAS-FL, and the object of our analysis, is an adaptive allocation rule that assigns sub-model capacities from both resource constraints and estimated data heterogeneity.
The rule is motivated by the following principle, which Section~\ref{subsec:rq2} puts to the test: \emph{clients with higher data heterogeneity require larger sub-model capacity to adequately capture their local data distribution, subject to their resource constraints}.

\begin{definition}[HAS-FL Allocation Rule]
\label{def:allocation}
Given resource constraints $\{p_i^{\max}\}_{i=1}^N$, heterogeneity estimates $\{\tilde{H}_i^{(t)}\}_{i=1}^N$, base capacity $p_{\min} \in (0, 1)$, and sensitivity parameter $\gamma > 0$, the adaptive allocation for client $i$ at round $t$ is:
\begin{equation}
\label{eq:allocation_rule}
p_i^{(t)} := \min\left(p_i^{\max}, \; p_{\min} + \gamma \cdot \frac{\tilde{H}_i^{(t)}}{\bar{H}^{(t)} + \epsilon}\right),
\end{equation}
where $\bar{H}^{(t)} := \frac{1}{N}\sum_{j=1}^{N} \tilde{H}_j^{(t)}$ is the average heterogeneity across all clients.
\end{definition}

The allocation rule in \eqref{eq:allocation_rule} has several properties that motivate it a priori. First, the minimum operator ensures that resource constraints are never violated, respecting the physical limitations of each device.
Second, the term $p_{\min}$ guarantees a baseline participation level for all clients, preventing any client from being effectively excluded.
Third, the heterogeneity-dependent term $\gamma \cdot \tilde{H}_i^{(t)}/\bar{H}^{(t)}$ allocates additional capacity in proportion to relative estimated heterogeneity, so that clients with the highest estimates receive larger sub-models.
The normalization by $\bar{H}^{(t)}$ makes the allocation robust to the absolute scale of heterogeneity measures across different training stages.
The sensitivity parameter $\gamma$ controls the strength of heterogeneity-aware allocation.
When $\gamma = 0$, all clients receive identical base capacity $p_{\min}$ (subject to resource constraints), recovering uniform allocation.
As $\gamma$ increases, the allocation becomes more responsive to heterogeneity differences.

The rule in \eqref{eq:allocation_rule} alone harbors a subtle failure mode. If every client---including those able to train the full model---is capped below full width, the coordinates beyond the largest allocation are trained by \emph{no} client: they remain at their random initialization yet continue to participate in inference, and the global model degrades progressively as the trained coordinates drift away from them. Related observations about unselected parameters have been made for static sub-model schemes~\citep{alam2022fedrolex}. HAS-FL therefore augments the allocation rule with a \emph{coverage guarantee}: after each adaptation step, clients whose resource constraint equals the largest constraint in the federation are assigned their full capacity $p_i^{\max}$. This ensures that every parameter of the global model keeps receiving updates whenever such a client is sampled, while leaving the heterogeneity-aware allocation of all remaining clients untouched. Section~\ref{subsec:ablation_coverage} quantifies the effect: without the guarantee, accuracy collapses; with it, HAS-FL matches static full-coverage allocation at lower average capacity.

The failure mode admits an exact statement.
\begin{proposition}[Frozen coordinates under capped allocation]
\label{prop:frozen}
Suppose every allocation satisfies $p_i^{(t)} \le \bar{p} < 1$ for all clients $i$ and rounds $t$, with the nested width masks of Algorithm~\ref{alg:mask}. Then every coordinate outside the leading $\bar{p}$-fraction slice of its layer satisfies $[\w^{(T)}]_j = [\w^{(0)}]_j$ for every $T$: it remains at its random initialization for the entire run while participating in every forward pass.
\end{proposition}
\begin{proof}
By nestedness, $[\m_i^{(t)}]_j = 0$ for every client $i$ whenever $j$ lies outside the leading $\bar{p}$ slice of its layer. The coordinate-wise aggregation rule (Algorithm~\ref{alg:hasfl}, line 18) sets $[\bar{\Delta}^{(t)}]_j = 0$ whenever no participant covers $j$, hence $[\w^{(t+1)}]_j = [\w^{(t)}]_j$ for all $t$, and induction from $\w^{(0)}$ completes the argument.
\end{proof}
The proposition turns the empirical collapse into a structural fact: an allocation policy avoids permanently frozen coordinates only if the union of client masks keeps covering every coordinate, and under nested masks the cheapest way to enforce this is to keep the largest-capacity clients at full width---precisely the coverage guarantee.

We now present the complete HAS-FL algorithm as outlined in Algorithm~\ref{alg:hasfl}, which integrates the heterogeneity estimation and adaptive allocation mechanisms into a federated optimization procedure.
The algorithm operates in rounds, with periodic updates to heterogeneity estimates and sub-model allocations.

The algorithm proceeds as follows.
At each round $t$, the server samples a subset $\cS^{(t)}$ of clients to participate (line 5).
For each selected client $i$, the server extracts the nested sub-model consisting of the leading $p_i^{(t)}$ fraction of each layer's channels (line 7).
The client performs $K$ local SGD steps on its local data (lines 9--13) and returns the parameter update $\Delta_i^{(t)}$ to the server (lines 14--15).
The server aggregates the received updates \emph{coordinate-wise} (lines 17--19): each parameter is averaged over exactly the clients that trained it, weighted by their local sample counts.
This replaces the scalar normalization used in earlier sub-model schemes; coordinates covered by few clients are neither diluted nor over-amplified, which is essential when allocation fractions differ across clients.
After a warm-up of $T_{\text{warm}}$ rounds, during which every client trains at its cap $p_i^{\max}$, the server updates heterogeneity estimates and reallocates sub-model capacities every $T_{\text{adapt}}$ rounds (lines 20--31).
The heterogeneity estimate $\hat{H}_i^{(t)}$ is computed from the divergence between client $i$'s update and the aggregated update, restricted to the coordinates that client $i$ actually trained (line 22), then smoothed using an exponential moving average (line 23) and, every $T_{\text{norm}}$ rounds, rescaled to unit mean (line 25).
Sub-model allocations are updated according to the allocation rule in Definition~\ref{def:allocation}, followed by the coverage guarantee that keeps the highest-capacity clients at full width (lines 26--30).
Algorithm~\ref{alg:mask} details the nested sub-model extraction procedure.

\begin{algorithm*}[!tbp]
\caption{HAS-FL: Heterogeneity-Aware Adaptive Sub-model Federated Learning}
\label{alg:hasfl}
\small
\begin{algorithmic}[1]
\Require Number of clients $N$, resource constraints $\{p_i^{\max}\}_{i=1}^N$, total rounds $T$, local steps $K$, learning rate $\eta$, adaptation interval $T_{\text{adapt}}$, warm-up length $T_{\text{warm}}$, normalization interval $T_{\text{norm}}$, smoothing parameter $\beta$, base capacity $p_{\min}$, sensitivity $\gamma$, client sampling rate $s$
\Ensure Trained global model $\w^{(T)}$
\State Initialize global model $\w^{(0)}$
\State Initialize heterogeneity estimates $\tilde{H}_i^{(0)} \gets 1$ for all $i \in [N]$
\State Initialize allocations $p_i^{(0)} \gets p_i^{\max}$ for all $i \in [N]$
\For{$t = 0, 1, \ldots, T-1$}
    \State $\cS^{(t)} \gets$ \textsc{SampleClients}$(N, s)$ \Comment{Sample subset of clients}
    \For{each client $i \in \cS^{(t)}$ \textbf{in parallel}}
        \State $\m_i^{(t)} \gets$ \textsc{ExtractSubmodel}$(d, p_i^{(t)})$ \Comment{Nested width mask: leading $p_i^{(t)}$ fraction of each layer}
        \State $\w_i^{(t,0)} \gets \m_i^{(t)} \odot \w^{(t)}$ \Comment{Extract sub-model}
        \For{$k = 0, 1, \ldots, K-1$}
            \State Sample mini-batch $\xi_i^{(t,k)} \sim \cD_i$
            \State $\g_i^{(t,k)} \gets \nabla \ell(\w_i^{(t,k)}; \xi_i^{(t,k)})$
            \State $\w_i^{(t,k+1)} \gets \w_i^{(t,k)} - \eta \cdot \m_i^{(t)} \odot \g_i^{(t,k)}$
        \EndFor
        \State $\Delta_i^{(t)} \gets \w_i^{(t,K)} - \m_i^{(t)} \odot \w^{(t)}$ \Comment{Compute masked update}
        \State Send $\Delta_i^{(t)}$ to server
    \EndFor
    \State \textbf{Server aggregation:}
    \State $[\bar{\Delta}^{(t)}]_j \gets \frac{\sum_{i \in \cS^{(t)}} n_i [\m_i^{(t)}]_j [\Delta_i^{(t)}]_j}{\sum_{i \in \cS^{(t)}} n_i [\m_i^{(t)}]_j}$ if $\sum_{i} [\m_i^{(t)}]_j > 0$, else $0$ \Comment{Coordinate-wise; $n_i$: sample count}
    \State $\w^{(t+1)} \gets \w^{(t)} + \bar{\Delta}^{(t)}$
    \If{$t \ge T_{\text{warm}}$ \textbf{and} $t \mod T_{\text{adapt}} = 0$} \Comment{Periodic adaptation after warm-up}
        \For{each client $i \in \cS^{(t)}$}
            \State $\hat{H}_i^{(t)} \gets \|\m_i^{(t)} \odot (\Delta_i^{(t)} - \bar{\Delta}^{(t)})\|^2 / (\|\m_i^{(t)} \odot \bar{\Delta}^{(t)}\|^2 + \epsilon)$
            \State $\tilde{H}_i^{(t)} \gets \beta \cdot \tilde{H}_i^{(t-T_{\text{adapt}})} + (1-\beta) \cdot \hat{H}_i^{(t)}$
        \EndFor
        \State \textbf{if} $t \mod T_{\text{norm}} = 0$ \textbf{then} $\tilde{H}_i^{(t)} \gets N \tilde{H}_i^{(t)} / \sum_{j} \tilde{H}_j^{(t)}$ for all $i$ \Comment{Rescale to unit mean}
        \State $\bar{H}^{(t)} \gets \frac{1}{N} \sum_{i=1}^{N} \tilde{H}_i^{(t)}$
        \For{each client $i \in [N]$}
            \State $p_i^{(t+1)} \gets \min\left(p_i^{\max}, \; p_{\min} + \gamma \cdot \tilde{H}_i^{(t)} / (\bar{H}^{(t)} + \epsilon)\right)$
        \EndFor
        \State $p_i^{(t+1)} \gets p_i^{\max}$ for all $i$ with $p_i^{\max} = \max_j p_j^{\max}$ \Comment{Coverage guarantee}
    \EndIf
\EndFor
\State \Return $\w^{(T)}$
\end{algorithmic}
\end{algorithm*}

\begin{algorithm*}[!tbp]
\caption{\textsc{ExtractSubmodel}: Nested Sub-model Extraction}
\label{alg:mask}
\begin{algorithmic}[1]
\Require Layer widths $\{C_\ell\}_{\ell=1}^{L}$, width fraction $p$
\Ensure Binary mask $\m \in \{0,1\}^d$
\For{each layer $\ell = 1, \ldots, L$}
    \State $c_\ell \gets \lceil p \cdot C_\ell \rceil$ \Comment{Number of retained channels}
    \State Set $m_j = 1$ for the parameters of the leading $c_\ell$ channels of layer $\ell$ (and their input connections from the retained channels of layer $\ell{-}1$); $m_j = 0$ otherwise
\EndFor
\State \Return $\m = (m_1, \ldots, m_d)$ \Comment{Masks are nested: $p \le p'$ implies $\m \subseteq \m'$}
\end{algorithmic}
\end{algorithm*}

Figure~\ref{fig:hasfl_overview} provides a graphical overview of the HAS-FL framework.
The server maintains a global model $\w^{(t)}$ and distributes client-specific sub-models based on adaptive allocations $\{p_i^{(t)}\}$.
Each client performs local training on its sub-model and returns updates $\Delta_i$ to the server.
The adaptation module, shown on the right, periodically estimates heterogeneity from update divergence, applies exponential smoothing, and computes new allocations from resource constraints and estimated heterogeneity.
As illustrated in the example, Client 1 has a high heterogeneity estimate ($H_1 = 2.1$) but limited resources ($p_1^{\max} = 0.3$), so it receives the maximum allowed capacity.
Client 2 has a low estimate ($H_2 = 0.3$); the base rule would assign it a small sub-model, but since it is the highest-capacity client in the federation, the coverage guarantee keeps it at full width ($p_2^{(t)} = 0.8$) so that every parameter of the global model continues to receive updates.
Client $N$ has a moderate estimate ($H_N = 1.2$) and receives its maximum allowed capacity ($p_N^{(t)} = 0.5 = p_N^{\max}$).

\begin{figure*}[!htbp]
\centering
\resizebox{\textwidth}{!}{\begin{tikzpicture}[
    node distance=1.5cm and 2cm,
    server/.style={rectangle, draw=serverblue, fill=lightblue, thick, minimum width=2.5cm, minimum height=1cm, rounded corners=3pt},
    client/.style={rectangle, draw=clientgreen, fill=lightgreen, thick, minimum width=2cm, minimum height=0.9cm, rounded corners=3pt},
    process/.style={rectangle, draw=updateorange, fill=lightorange, thick, minimum width=2.3cm, minimum height=0.7cm, rounded corners=3pt, font=\small},
    arrow/.style={-{Stealth[length=2mm]}, thick},
    dashedarrow/.style={-{Stealth[length=2mm]}, thick, dashed},
    label/.style={font=\scriptsize},
]

\node[server] (server) {\textbf{Server}};
\node[above=0.1cm of server, font=\small\bfseries, color=serverblue] {Global Model $\w^{(t)}$};

\node[client, below left=2cm and 3cm of server] (client1) {Client 1};
\node[client, below=2cm of server] (client2) {Client 2};
\node[client, below right=2cm and 1.2cm of server] (client3) {Client $N$};
\node[below right=2cm and -0.2cm of server] (dots) {$\cdots$};
\node[below=0.25cm of client1, font=\scriptsize] (res1) {$p_1^{\max} = 0.3$};
\node[below=0.25cm of client2, font=\scriptsize] (res2) {$p_2^{\max} = 0.8$};
\node[below=0.25cm of client3, font=\scriptsize] (res3) {$p_N^{\max} = 0.5$};
\node[below=0.08cm of res1, font=\scriptsize, color=heterored] (het1) {$H_1 = 2.1$ (high)};
\node[below=0.08cm of res2, font=\scriptsize, color=clientgreen] (het2) {$H_2 = 0.3$ (low)};
\node[below=0.08cm of res3, font=\scriptsize, color=updateorange] (het3) {$H_N = 1.2$ (med)};
\node[below=0.08cm of het1, font=\scriptsize\bfseries] (alloc1) {$p_1^{(t)} = 0.3$};
\node[below=0.08cm of het2, font=\scriptsize\bfseries] (alloc2) {$p_2^{(t)} = 0.8$};
\node[below=0.08cm of het3, font=\scriptsize\bfseries] (alloc3) {$p_N^{(t)} = 0.5$};

\draw[arrow, serverblue] (server.south west) -- node[left, label, pos=0.4] {$\m_1 \odot \w$} (client1.north);
\draw[arrow, serverblue] (server.south) -- node[right, label, pos=0.4] {$\m_2 \odot \w$} (client2.north);
\draw[arrow, serverblue] ([xshift=0.3cm]server.south) -- node[right, label, pos=0.4] {$\m_N \odot \w$} (client3.north);
\draw[arrow, clientgreen] (client1.north east) -- node[above left, label, pos=0.6] {$\Delta_1$} ([xshift=-0.6cm]server.south);
\draw[arrow, clientgreen] (client2.north) -- node[left, label, pos=0.6] {$\Delta_2$} ([xshift=0cm]server.south);
\draw[arrow, clientgreen] (client3.north west) -- node[above right, label, pos=0.6] {$\Delta_N$} ([xshift=0.5cm]server.south);
\node[process, right=4.5cm of server] (estimate) {Estimate $\hat{H}_i^{(t)}$};
\node[process, below=0.6cm of estimate] (smooth) {Smooth $\tilde{H}_i^{(t)}$};
\node[process, below=0.6cm of smooth] (allocate) {Allocate $p_i^{(t+1)}$};
\draw[arrow] (server.east) -- (estimate.west);
\draw[arrow] (estimate.south) -- (smooth.north);
\draw[arrow] (smooth.south) -- (allocate.north);
\draw[dashedarrow] (allocate.west) -- +(-0.8cm,0) |- ([yshift=-0.2cm]server.east);

\node[above=0.25cm of estimate, font=\small\bfseries, color=updateorange] {Adaptation Module};
\begin{scope}[on background layer]
\node[fit=(estimate)(smooth)(allocate), draw=updateorange, dashed, rounded corners=5pt, inner sep=0.25cm] {};
\end{scope}

\end{tikzpicture}}
\caption{Overview of the HAS-FL framework.}
\label{fig:hasfl_overview}
\end{figure*}
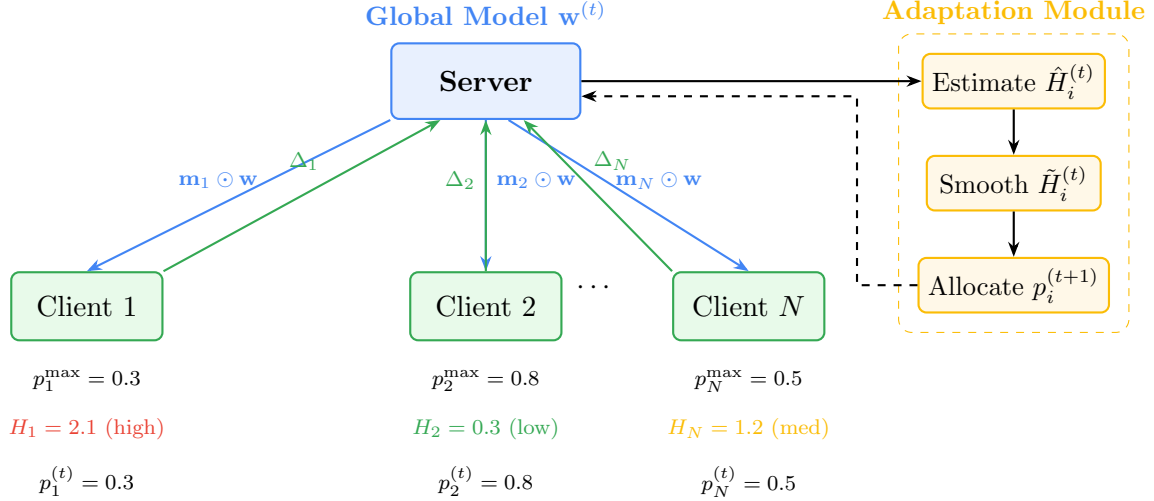

With the parameters used in our experiments, $p_{\min} = 0.4$ and $\gamma = 0.25$, the allocation rule in Equation~\eqref{eq:allocation_rule} behaves as follows. Allocated capacity grows linearly with relative heterogeneity $\tilde{H}_i/\bar{H}$, from $p_{\min}$ at $\tilde{H}_i = 0$, until it reaches the client's cap $p_i^{\max}$, after which it stays constant. The cap is reached at $\tilde{H}_i/\bar{H} = (p_i^{\max} - p_{\min})/\gamma$. For the four capacity tiers in our experiments, this happens at $\tilde{H}_i/\bar{H} = 2.4$ for $p_i^{\max} = 1.0$, at $1.4$ for $0.75$, and at $0.4$ for $0.5$. Clients in the $0.25$ tier are fixed at $0.25$ whatever their estimate, because the cap is below $p_{\min}$. This is the intended behaviour of the rule: clients with high estimated heterogeneity receive larger sub-models when resources permit, while resource-constrained clients are capped regardless of their estimate. It also means that for most clients the heterogeneity estimate has a narrow range in which to act, since a client at half width reaches its cap at an estimate only 40\% of the population average.

\subsection{Communication and Computation Complexity}
\label{subsec:complexity}

We analyze the communication and computation costs of HAS-FL compared to standard federated learning approaches.
In each round, client $i$ receives a sub-model of $\rho_i^{(t)} \cdot d$ parameters, where $\rho_i^{(t)} \approx (p_i^{(t)})^2$ is its parameter fraction, and sends back an update of the same size.
The total communication per round is:
\begin{equation}
\label{eq:comm_cost}
C_{\text{comm}}^{(t)} = 2 \sum_{i \in \cS^{(t)}} \rho_i^{(t)} \cdot d = 2d \sum_{i \in \cS^{(t)}} \rho_i^{(t)} \approx 2d \sum_{i \in \cS^{(t)}} \bigl(p_i^{(t)}\bigr)^2.
\end{equation}
Compared to full-model training where $C_{\text{full}} = 2|\cS^{(t)}| \cdot d$, HAS-FL achieves a communication reduction factor of $|\cS^{(t)}|/\sum_{i \in \cS^{(t)}} \rho_i^{(t)}$.
Because $\rho_i \approx p_i^2$, the saving grows quadratically as widths shrink, which matters most for resource-constrained edge devices.
The local computation cost for client $i$ scales with sub-model size as $\cO(\rho_i^{(t)} \cdot d \cdot K \cdot B)$, where $K$ is the number of local steps and $B$ is the batch size.
Resource-constrained clients with small $p_i^{\max}$ thus perform quadratically less computation, enabling their participation in the training process.
The adaptation module introduces additional server-side computation for heterogeneity estimation and allocation updates.
Computing $\hat{H}_i^{(t)}$ requires $\cO(d)$ operations per client, and the allocation rule evaluation is $\cO(N)$.
Since adaptation occurs every $T_{\text{adapt}}$ rounds, the amortized overhead per round is $\cO((d \cdot |\cS| + N)/T_{\text{adapt}})$, which is negligible compared to the aggregation cost of $\cO(d \cdot |\cS|)$.

Figure~\ref{fig:cost_model} instantiates these costs for the three architectures used in our experiments.
Because width slicing shrinks \emph{both} the input and output channels of every layer, per-client cost scales quadratically rather than linearly with $p$: a quarter-width client trains only 6--7\% of the full model's parameters and multiply--accumulate operations, and a half-width client roughly 25\%.
This quadratic law is what makes sub-model training attractive for edge deployments---the weakest capacity tier pays an order of magnitude less computation and upload bandwidth than the strongest---but it is also the root of the estimation problem studied in Section~\ref{subsec:rq2}: sub-model updates differ across capacity tiers by an order of magnitude in dimensionality and norm, dwarfing any data-driven variation.

\begin{figure*}[!tbp]
\centering
\begin{tikzpicture}
\begin{groupplot}[
    group style={group size=1 by 3, vertical sep=13mm, x descriptions at=edge bottom},
    width=\textwidth, height=5.4cm,
    ymode=log, ymin=0.012, ymax=1200,
    ytick={0.1,1,10,100},
    xmin=0.2, xmax=1.05,
    xtick={0.25,0.5,0.75,1.0},
    xticklabels={0.25,0.50,0.75,1.00},
    ylabel={Cost (log scale)},
    tick label style={font=\footnotesize},
    label style={font=\footnotesize},
    title style={font=\footnotesize\sffamily\bfseries, at={(0,1)}, anchor=south west, yshift=-1mm},
    grid=major, grid style={draw=gray!25},
    point meta=rawy,
    /pgf/number format/fixed, /pgf/number format/precision=2,
    legend columns=3,
    legend style={font=\footnotesize, draw=none, at={(0.5,1.17)}, anchor=south, /tikz/every even column/.append style={column sep=4mm}},
]
\nextgroupplot[title={(a) CIFAR-10 CNN},]
\addplot[color=black,thick,solid,mark=*,mark size=2.2pt,mark options={solid,fill=black},nodes near coords,every node near coord/.append style={font=\scriptsize,color=black,anchor=north,yshift=-1pt}] coordinates {(0.25,0.21) (0.5,0.82) (0.75,1.83) (1.0,3.25)};\addlegendentry{Parameters (M)}
\addplot[color=black,thick,dashed,mark=square*,mark size=2.2pt,mark options={solid,fill=white},nodes near coords,every node near coord/.append style={font=\scriptsize,color=black,anchor=south,yshift=1pt}] coordinates {(0.25,10.0) (0.5,39.2) (0.75,87.4) (1.0,154.9)};\addlegendentry{MACs per sample (M)}
\addplot[color=black,thick,densely dotted,mark=triangle*,mark size=2.8pt,mark options={solid,fill=gray!60},nodes near coords,every node near coord/.append style={font=\scriptsize,color=black,anchor=south,yshift=1pt}] coordinates {(0.25,0.82) (0.5,3.26) (0.75,7.32) (1.0,13.0)};\addlegendentry{Upload (MB)}

\nextgroupplot[title={(b) EMNIST CNN},]
\addplot[color=black,thick,solid,mark=*,mark size=2.2pt,mark options={solid,fill=black},nodes near coords,every node near coord/.append style={font=\scriptsize,color=black,anchor=north,yshift=-1pt}] coordinates {(0.25,0.21) (0.5,0.83) (0.75,1.85) (1.0,3.28)};
\addplot[color=black,thick,dashed,mark=square*,mark size=2.2pt,mark options={solid,fill=white},nodes near coords,every node near coord/.append style={font=\scriptsize,color=black,anchor=south,yshift=1pt}] coordinates {(0.25,7.5) (0.5,29.7) (0.75,66.6) (1.0,118.2)};
\addplot[color=black,thick,densely dotted,mark=triangle*,mark size=2.8pt,mark options={solid,fill=gray!60},nodes near coords,every node near coord/.append style={font=\scriptsize,color=black,anchor=south,yshift=1pt}] coordinates {(0.25,0.85) (0.5,3.31) (0.75,7.4) (1.0,13.1)};

\nextgroupplot[title={(c) Shakespeare LSTM},xlabel={Sub-model width $p$},]
\addplot[color=black,thick,solid,mark=*,mark size=2.2pt,mark options={solid,fill=black},nodes near coords,every node near coord/.append style={font=\scriptsize,color=black,anchor=north,yshift=-1pt}] coordinates {(0.25,0.06) (0.5,0.22) (0.75,0.47) (1.0,0.82)};
\addplot[color=black,thick,dashed,mark=square*,mark size=2.2pt,mark options={solid,fill=white},nodes near coords,every node near coord/.append style={font=\scriptsize,color=black,anchor=south,yshift=1pt}] coordinates {(0.25,4.1) (0.5,16.1) (0.75,35.9) (1.0,63.6)};
\addplot[color=black,thick,densely dotted,mark=triangle*,mark size=2.8pt,mark options={solid,fill=gray!60},nodes near coords,every node near coord/.append style={font=\scriptsize,color=black,anchor=south,yshift=1pt}] coordinates {(0.25,0.24) (0.5,0.86) (0.75,1.88) (1.0,3.3)};
\end{groupplot}
\end{tikzpicture}
\caption{Per-client cost against sub-model width $p$ (log scale) for (a) the CIFAR-10 CNN, (b) the EMNIST CNN and (c) the Shakespeare LSTM. In all three models, parameters, MACs per sample and upload size grow roughly as $p^2$.}
\label{fig:cost_model}
\end{figure*}

\section{Experiments}
\label{sec:experiments}

We test the premises of HAS-FL on three benchmarks.
The evaluation addresses three research questions: (1) Does heterogeneity-aware adaptive allocation change convergence or final accuracy relative to resource-only allocation strategies?
(2) Can client-level data heterogeneity be estimated from the sub-model updates the server observes, and do the resulting allocations track it?
(3) How do sub-model methods, HAS-FL included, compare against full-model federated optimization methods designed for statistical heterogeneity?
The following subsections present our experimental setup, followed by systematic investigation of each research question with evidence drawn from multiple benchmark datasets.

\subsection{Experimental Setup}
\label{subsec:setup}

The experiments use three benchmarks: two image-classification datasets with synthetic label skew, which differ in scale and class count, and one text dataset with natural client partitions.
The CIFAR-10 dataset \citep{krizhevsky2009learning} consists of 60,000 color images of size $32 \times 32$ pixels distributed across 10 mutually exclusive classes, with 50,000 images designated for training and 10,000 for testing.
We partition the training data across clients using a Dirichlet distribution with concentration parameter $\alpha = 0.3$, creating a challenging non-IID setting where each client possesses a highly skewed subset of the label space.
This dataset serves as our primary benchmark due to its widespread adoption in federated learning research, enabling direct comparison with existing methods while allowing controlled experimentation with varying degrees of label skew.
The EMNIST dataset comprises 697,932 grayscale images of handwritten characters distributed across 62 classes including digits, uppercase letters, and lowercase letters.
This represents a 14-fold increase in data volume and 6-fold increase in class count compared to CIFAR-10, testing scalability and performance under increased classification complexity.
We apply the same Dirichlet partitioning strategy ($\alpha = 0.3$) to maintain consistency with our CIFAR-10 experiments while testing whether the findings carry over to a substantially larger and more complex classification task.
To complement the label-skew heterogeneity induced by Dirichlet partitioning, we additionally evaluate on a benchmark with \emph{natural} client partitions: next-character prediction on the complete works of Shakespeare, following the LEAF benchmark's construction~\citep{caldas2018leaf}.
Each client corresponds to one speaking role, so heterogeneity arises from genuine stylistic and lexical differences between characters rather than from a synthetic sampling procedure; the twenty roles with the largest amount of dialogue form the federation.
Text is modeled at the character level (vocabulary of 95 symbols) with sequences of 80 characters predicting the following character, and the model is a two-layer LSTM with 256 hidden units per layer and an 8-dimensional character embedding.
Sub-model extraction for the LSTM slices the leading fraction of hidden units \emph{within each gate block}, preserving the recurrent structure of nested sub-models; the embedding and output bias are shared in full by every client.
Training uses one local epoch, $\eta = 0.1$, and gradient-norm clipping at 5; all other protocol elements are identical to the image experiments.

Table~\ref{tab:datasets} summarizes the key characteristics of these datasets.

\begin{table*}[!tbp]
\centering
\caption{Datasets used in the experiments.}
\label{tab:datasets}
\footnotesize
\begin{tabular}{lccccc}
\toprule
Dataset & Task & Classes & Train Size & Test Size & Input Dim \\
\midrule
CIFAR-10 & Image & 10 & 50,000 & 10,000 & $32 \times 32 \times 3$ \\
EMNIST & Image & 62 & 697,932 & 116,323 & $28 \times 28 \times 1$ \\
Shakespeare & Text & 95 & \multicolumn{2}{c}{20 role-clients} & $80$ chars \\
\bottomrule
\end{tabular}
\end{table*}

The federated learning setup consists of $N=20$ clients, with 10 clients randomly sampled to participate in each communication round.
Resource heterogeneity is simulated by assigning each client a maximum capacity constraint $p_i^{\max} \in \{0.25, 0.5, 0.75, 1.0\}$ according to a distribution in which resource-constrained devices are more prevalent. Each seed redraws this assignment, so the realized mean cap varies by seed; over our three seeds it averages approximately $0.56$, which is the average capacity of HeteroFL, the method that trains every client at its cap (Table~\ref{tab:ablation_coverage}).
Training proceeds for $T=200$ rounds with learning rate $\eta=0.01$, decay factor 0.995, and five local epochs for CIFAR-10 (one local epoch for EMNIST, following standard practice for large federated datasets); the number of local steps $K$ in Algorithm~\ref{alg:hasfl} is the number of mini-batches a client processes in these epochs.
Local optimization uses SGD with momentum 0.9, weight decay $10^{-4}$, and gradient-norm clipping at 10. All experiments are repeated over three random seeds and reported as mean $\pm$ standard deviation.
The HAS-FL hyperparameters are set to $p_{\min} = 0.4$, $\gamma = 0.25$, and $\beta = 0.9$, with adaptation occurring every $T_{\text{adapt}} = 5$ rounds following a warm-up of $T_{\text{warm}} = 10$ rounds, and the coverage guarantee of Section~\ref{subsec:allocation} enabled.
We set the normalization interval to $T_{\text{norm}} = 20$ rounds.

We compare HAS-FL against seven baseline algorithms spanning sub-model and full-model approaches.
Among sub-model methods, HeteroFL \citep{diao2021heterofl} assigns sub-model widths based solely on resource constraints without considering data heterogeneity, FjORD~\citep{horvath2021fjord} trains nested sub-networks by ordered dropout, which we implement by resampling each client's width tier at random from the tiers no larger than $p_i^{\max}$ in every round, without the self-distillation component of the original method, and the Uniform baseline assigns fixed allocation $p_i = \min(0.5, p_i^{\max})$ to all clients regardless of their characteristics.
A Random-budget control assigns each client a random time-varying allocation drawn to match HAS-FL's average capacity budget, isolating the contribution of heterogeneity-awareness from mere adaptivity.
Among full-model methods, FedAvg \citep{mcmahan2017communication} serves as the standard baseline, SCAFFOLD \citep{karimireddy2020scaffold} employs control variates to correct client drift, and FedProx \citep{li2020fedprox} adds proximal regularization to constrain local updates.
All methods use identical data partitions, resource constraints, and participation ratios, and, with one exception, identical optimization hyperparameters. The exception is SCAFFOLD, whose local steps use vanilla SGD without momentum or weight decay because its control-variate update is not valid under momentum (Section~\ref{subsec:rq3}); it is otherwise evaluated under the same $N=20$ client configuration as every other method.

\subsection{Heterogeneity-Aware vs. Resource-Only Allocation}
\label{subsec:rq1}

Our first research question examines whether incorporating data heterogeneity estimates into capacity allocation decisions changes performance relative to strategies that consider only resource constraints.
We address this by comparing HAS-FL against two resource-only approaches: the Uniform baseline, which assigns identical capacity to all clients within their resource limits, and FjORD, which assigns capacity tiers randomly from the available options without considering client-specific data characteristics.
The comparison asks what heterogeneity-aware allocation contributes relative to policies that use no data signal; a control matched exactly on capacity budget follows in Section~\ref{subsec:rq2}.
Table~\ref{tab:rq1_comparison} presents the comparison, and its message is central to this paper's thesis: heterogeneity-aware allocation does \emph{not} separate from random tier assignment.
On CIFAR-10, HAS-FL reaches 74.5\% $\pm$ 5.6 final accuracy while FjORD reaches 75.5\% $\pm$ 3.8---a difference well within seed-to-seed variation---even though FjORD uses \emph{less} average capacity (0.41 vs.\ 0.52).
On EMNIST the pattern repeats: 77.7\% $\pm$ 4.6 for HAS-FL against 79.6\% $\pm$ 3.3 for FjORD at 0.40 average capacity.
Both coverage-preserving methods dominate the Uniform baseline, and the margin widens dramatically with task complexity: Uniform ends at 54.5\% $\pm$ 2.3 on CIFAR-10 but collapses to 9.4\% $\pm$ 3.8 on EMNIST---a fall of 70--75 points from its own peak of 82.2\% $\pm$ 1.4, to barely above the 1.6\% chance level of the 62-class task---confirming that the decisive design element is coverage, not the allocation criterion.
Following recent recommendations on evaluation practice, we emphasize final rather than peak accuracy and report mean $\pm$ standard deviation over three seeds; seed-to-seed variation is substantial (up to six percentage points) because each seed redraws both the data partition and the device capacity assignment, which underscores the importance of multi-seed reporting. Because competing policies share identical partitions and capacity assignments within each seed, the matched-budget comparisons are additionally assessed through paired per-seed differences, which support the same conclusions as the aggregate means.

\begin{table}[!htbp]
\centering
\caption{Sub-model allocation strategies (mean $\pm$ std over three seeds).}
\label{tab:rq1_comparison}
\footnotesize
\setlength{\tabcolsep}{3.5pt}
\begin{tabular}{llccc}
\toprule
Dataset & Algorithm & Best Acc. & Final Acc. & Avg. $p$ \\
\midrule
\multirow{3}{*}{CIFAR-10}
& HAS-FL & 76.7 $\pm$ 4.9\% & 74.5 $\pm$ 5.6\% & 0.52 \\
& FjORD & 78.6 $\pm$ 3.9\% & 75.5 $\pm$ 3.8\% & 0.41 \\
& Uniform & 63.6 $\pm$ 1.6\% & 54.5 $\pm$ 2.3\% & 0.44 \\
\midrule
\multirow{3}{*}{EMNIST}
& HAS-FL & 82.0 $\pm$ 1.8\% & 77.7 $\pm$ 4.6\% & 0.52 \\
& FjORD & 84.0 $\pm$ 0.5\% & 79.6 $\pm$ 3.3\% & 0.40 \\
& Uniform & 82.2 $\pm$ 1.4\% & 9.4 $\pm$ 3.8\% & 0.44 \\
\bottomrule
\end{tabular}
\end{table}

The convergence trajectories in Figure~\ref{fig:rq1_convergence} reveal qualitatively different behaviors behind these aggregates.
HAS-FL and FjORD improve throughout training, with FjORD oscillating visibly because its random tier assignment changes each client's sub-model width from round to round.
The Uniform baseline follows a characteristic rise-and-decline pattern on every seed: it reaches the low-to-mid 60\% range by round $\sim$110 and then deteriorates steadily, ending 6--11 points below its own peak (65.4\%$\to$54.2\%, 62.4\%$\to$52.3\%, and 62.9\%$\to$56.9\% across the three seeds).
This is the stale-parameter failure mode analyzed in Section~\ref{subsec:ablation_coverage}: under the fixed allocation $p_i = \min(0.5, p_i^{\max})$ no client ever trains the outer half of each layer's channels, and the untrained coordinates progressively corrupt inference as the trained portion of the network drifts away from them.
On CIFAR-10 the decline is gradual rather than abrupt: coordinate-wise aggregation partially contains the damage, but only an allocation policy that preserves coverage eliminates it.
On EMNIST even this containment fails: after peaking above 80\%, Uniform decays to near-chance accuracy with the test loss settling at $\ln 62 \approx 4.1$---the signature of outputs degraded to random guessing---showing that the severity of the coverage failure grows with the number of classes.

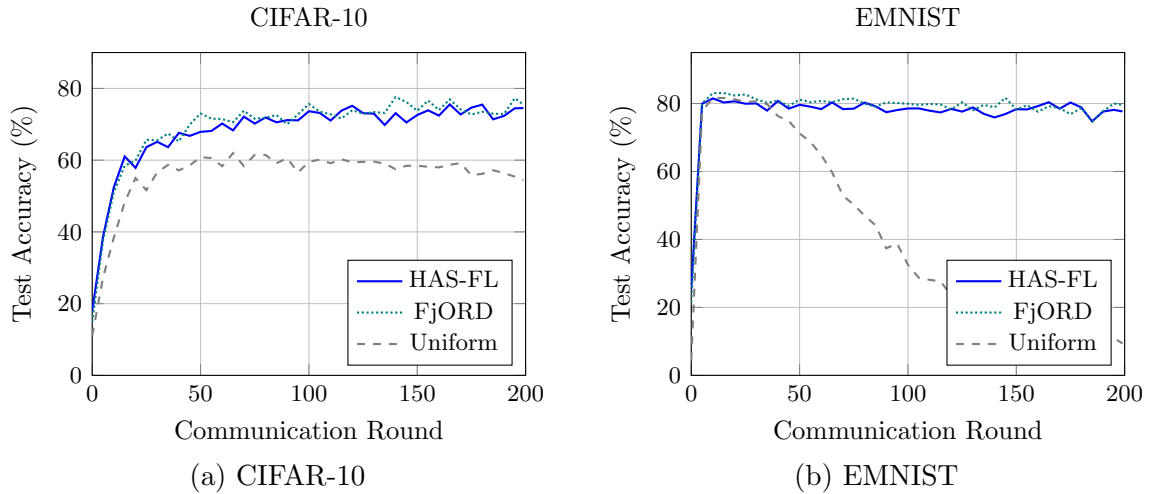
\begin{figure*}[!htbp]
    \centering
    \begin{minipage}{0.48\textwidth}
        \centering
        \begin{tikzpicture}
        \begin{axis}[
            width=\textwidth,
            height=0.8\textwidth,
            xlabel={Communication Round},
            ylabel={Test Accuracy (\%)},
            xmin=0, xmax=200,
            ymin=0, ymax=90,
            grid=both,
            grid style={line width=0.1pt, draw=gray!30},
            major grid style={line width=0.2pt, draw=gray!50},
            legend pos=south east,
            legend style={font=\footnotesize},
            label style={font=\small},
            tick label style={font=\footnotesize},
            title={CIFAR-10},
            title style={font=\small},
        ]
        \addplot[blue, thick] coordinates {
    (0, 17.95) (5, 38.72) (10, 52.46) (15, 61.03) (20, 57.85) (25, 63.64) (30, 65.11) (35, 63.62) (40, 67.62) (45, 66.77) (50, 67.89) (55, 68.14) (60, 70.21) (65, 68.30) (70, 72.14) (75, 70.19) (80, 71.90) (85, 70.55) (90, 71.16) (95, 71.09) (100, 73.64) (105, 73.16) (110, 71.04) (115, 73.79) (120, 75.16) (125, 73.10) (130, 72.96) (135, 69.80) (140, 73.10) (145, 70.52) (150, 72.59) (155, 73.89) (160, 72.42) (165, 75.54) (170, 72.74) (175, 74.65) (180, 75.48) (185, 71.38) (190, 72.33) (195, 74.45) (199, 74.52)
};
        \addlegendentry{HAS-FL}
        
        \addplot[teal, thick, densely dotted] coordinates {
    (0, 14.93) (5, 37.63) (10, 51.08) (15, 58.47) (20, 59.90) (25, 65.66) (30, 65.54) (35, 67.43) (40, 65.23) (45, 69.92) (50, 73.00) (55, 71.57) (60, 71.43) (65, 70.59) (70, 73.78) (75, 71.55) (80, 71.78) (85, 72.60) (90, 70.04) (95, 72.78) (100, 75.66) (105, 73.51) (110, 72.77) (115, 71.61) (120, 73.93) (125, 72.96) (130, 73.40) (135, 73.16) (140, 77.58) (145, 76.20) (150, 73.76) (155, 76.62) (160, 73.97) (165, 76.95) (170, 74.16) (175, 72.72) (180, 73.49) (185, 72.95) (190, 72.86) (195, 77.26) (199, 75.48)
};
        \addlegendentry{FjORD}
        
        \addplot[gray, thick, dashed] coordinates {
    (0, 11.07) (5, 27.24) (10, 38.47) (15, 48.27) (20, 55.00) (25, 51.64) (30, 56.60) (35, 58.70) (40, 57.15) (45, 58.53) (50, 60.88) (55, 60.56) (60, 58.31) (65, 61.98) (70, 58.33) (75, 61.48) (80, 61.45) (85, 59.13) (90, 60.57) (95, 56.56) (100, 59.58) (105, 60.12) (110, 59.14) (115, 60.26) (120, 59.39) (125, 59.52) (130, 59.59) (135, 59.02) (140, 57.48) (145, 58.43) (150, 58.46) (155, 58.15) (160, 57.99) (165, 58.72) (170, 59.18) (175, 55.78) (180, 56.19) (185, 57.18) (190, 56.39) (195, 55.40) (199, 54.45)
};
        \addlegendentry{Uniform}
        \end{axis}
        \end{tikzpicture}
        \centerline{(a) CIFAR-10}
    \end{minipage}
    \hfill
    \begin{minipage}{0.48\textwidth}
        \centering
        \begin{tikzpicture}
        \begin{axis}[
            width=\textwidth,
            height=0.8\textwidth,
            xlabel={Communication Round},
            ylabel={Test Accuracy (\%)},
            xmin=0, xmax=200,
            ymin=0, ymax=95,
            grid=both,
            grid style={line width=0.1pt, draw=gray!30},
            major grid style={line width=0.2pt, draw=gray!50},
            label style={font=\small},
            tick label style={font=\footnotesize},
            title={EMNIST},
            title style={font=\small},
        ]
        \addplot[blue, thick] coordinates {
    (0, 25.39) (5, 79.94) (10, 81.50) (15, 80.31) (20, 80.65) (25, 79.94) (30, 80.03) (35, 77.92) (40, 80.75) (45, 78.49) (50, 79.62) (55, 79.07) (60, 78.36) (65, 80.32) (70, 78.35) (75, 78.49) (80, 80.27) (85, 79.19) (90, 77.44) (95, 78.08) (100, 78.54) (105, 78.56) (110, 77.89) (115, 77.38) (120, 78.42) (125, 77.62) (130, 78.84) (135, 76.96) (140, 75.90) (145, 76.92) (150, 78.34) (155, 78.26) (160, 79.25) (165, 80.38) (170, 78.46) (175, 80.30) (180, 78.86) (185, 74.75) (190, 77.63) (195, 78.14) (199, 77.67)
};
        
        \addplot[teal, thick, densely dotted] coordinates {
    (0, 21.59) (5, 80.04) (10, 83.09) (15, 83.07) (20, 82.37) (25, 82.68) (30, 81.24) (35, 80.22) (40, 80.83) (45, 79.30) (50, 81.16) (55, 80.38) (60, 80.75) (65, 80.37) (70, 81.26) (75, 81.43) (80, 79.95) (85, 79.09) (90, 80.32) (95, 80.15) (100, 79.93) (105, 79.58) (110, 79.86) (115, 79.72) (120, 78.25) (125, 80.37) (130, 78.01) (135, 79.62) (140, 78.88) (145, 81.69) (150, 78.23) (155, 79.51) (160, 77.60) (165, 79.22) (170, 78.64) (175, 76.87) (180, 78.65) (185, 74.73) (190, 77.70) (195, 79.99) (199, 79.59)
};
        
        \addplot[gray, thick, dashed] coordinates {
    (0, 4.43) (5, 77.85) (10, 81.63) (15, 81.64) (20, 81.26) (25, 80.44) (30, 80.65) (35, 79.84) (40, 76.28) (45, 74.88) (50, 71.21) (55, 68.68) (60, 64.82) (65, 59.75) (70, 52.92) (75, 50.19) (80, 46.99) (85, 44.20) (90, 37.40) (95, 38.78) (100, 32.52) (105, 28.71) (110, 28.10) (115, 27.67) (120, 23.56) (125, 21.01) (130, 21.69) (135, 19.38) (140, 17.57) (145, 16.53) (150, 13.07) (155, 16.28) (160, 15.07) (165, 13.82) (170, 14.52) (175, 11.40) (180, 12.52) (185, 10.62) (190, 10.56) (195, 11.01) (199, 9.37)
};
        \end{axis}
        \end{tikzpicture}
        \centerline{(b) EMNIST}
    \end{minipage}
    \caption{Convergence of sub-model allocation strategies on CIFAR-10 (a) and EMNIST (b).}
    \label{fig:rq1_convergence}
\end{figure*}

The Uniform baseline's failure has a precise mechanistic explanation. Under the fixed allocation $p_i = \min(0.5, p_i^{\max})$, no client ever trains more than half of each layer's channels: the remaining channels are updated by no client, stay at their random initialization, and continue to participate in every forward pass. As training progresses, the trained portion of the network drifts away from these stale coordinates and their contribution becomes increasingly destructive---gradual degradation on CIFAR-10 and outright collapse on the 62-class EMNIST task, where the mismatch compounds faster. The coverage-guarantee ablation in Section~\ref{subsec:ablation_coverage} reproduces this failure mode in a controlled setting and shows that preserving full parameter coverage eliminates it. Which parameters an allocation policy leaves untrained thus matters as much as how much total capacity it uses.

\subsection{Ablation: The Coverage Guarantee}
\label{subsec:ablation_coverage}

To isolate the effect of the coverage guarantee introduced in Section~\ref{subsec:allocation}, we compare three configurations under identical data partitions, seeds, and training budgets: (i) full HAS-FL; (ii) HAS-FL with the coverage guarantee disabled, i.e., the allocation rule of Definition~\ref{def:allocation} applied verbatim; and (iii) HeteroFL, whose static resource-only allocations provide full coverage by construction. Table~\ref{tab:ablation_coverage} reports the results over three seeds on both image datasets.

\begin{table*}[!htbp]
\centering
\caption{Coverage-guarantee ablation (mean $\pm$ std over three seeds).}
\label{tab:ablation_coverage}
\footnotesize
\begin{tabular}{llcccc}
\toprule
Dataset & Configuration & Best Acc. & Final Acc. & Peak$\to$final $\Delta$ & Avg.\ $p$ \\
\midrule
\multirow{3}{*}{CIFAR-10}
& HAS-FL, guarantee on & 76.7 $\pm$ 4.9\% & 74.5 $\pm$ 5.6\% & $-2.2$ & 0.52 \\
& HAS-FL, guarantee off & 57.2 $\pm$ 2.4\% & 40.4 $\pm$ 10.2\% & $-16.8$ & 0.48 \\
& HeteroFL (full coverage) & 76.3 $\pm$ 4.9\% & 75.0 $\pm$ 5.0\% & $-1.2$ & 0.56 \\
\midrule
\multirow{2}{*}{EMNIST}
& HAS-FL, guarantee on & 82.0 $\pm$ 1.8\% & 77.7 $\pm$ 4.6\% & $-4.4$ & 0.52 \\
& HAS-FL, guarantee off & 82.4 $\pm$ 2.3\% & 70.6 $\pm$ 9.4\% & $-11.8$ & 0.47 \\
\bottomrule
\end{tabular}
\end{table*}

Disabling the guarantee leaves the largest allocation in the federation at $0.58$--$0.69$ rather than $1.0$, so roughly a third of every layer's channels are trained by no client. The consequence separates cleanly into two distinct effects, and the two datasets isolate them.

The first is a \emph{capacity} effect, visible in peak accuracy. On CIFAR-10, best accuracy falls from $76.7\%$ to $57.2\%$: a two-thirds-width network simply cannot represent the task as well. On EMNIST, by contrast, peak accuracy is unchanged ($82.4\%$ without the guarantee versus $82.0\%$ with it), because a two-thirds-width model already saturates that task. Lost coverage costs accuracy only to the extent that the stranded capacity was needed.

The second is a \emph{stability} effect, visible in the gap between peak and final accuracy, and it appears on both datasets regardless of whether capacity binds. With the guarantee, models end training within a few points of their peak ($-2.2$ on CIFAR-10, $-4.4$ on EMNIST, and $-1.2$ for fully covered HeteroFL). Without it, the decay grows to $-16.8$ and $-11.8$ respectively, and the seed-to-seed spread of final accuracy roughly doubles ($\pm 10.2$ and $\pm 9.4$, against $\pm 5.6$ and $\pm 4.6$). This is the mechanism of Proposition~\ref{prop:frozen} acting over time: coordinates frozen at initialization drift ever further from the trained portion of the network, and their contribution to inference becomes progressively more destructive.

The magnitude is also strongly seed-dependent. Removing the guarantee costs $19$, $34$, and $49$ points of final accuracy on the three CIFAR-10 seeds, and $0.3$, $4.7$, and $16.1$ points on EMNIST. A single-seed ablation on EMNIST could therefore have shown almost no effect at all; only the three-seed protocol reveals that the failure is severe but intermittent, triggered by particular combinations of data partition and capacity assignment.

Taken together, the ablation supports a more precise claim than ``coverage is necessary.'' Uncovered parameters always destabilise training, and they additionally cost accuracy when the task needs the width they occupy. This also reconciles the Uniform baseline's behaviour across benchmarks: it caps every client at half width, stranding both more parameters and---critically---the outer channels of the classifier, which is why it collapses outright on the 62-class EMNIST task (Section~\ref{subsec:rq1}) yet survives on Shakespeare, where the output projection is shared in full by every client (Section~\ref{subsec:shakespeare}).

\subsection{Validating Update Divergence Against Ground-Truth Heterogeneity}
\label{subsec:rq2}

Our second research question is whether the divergence estimates of Section~\ref{subsec:heterogeneity_measure} actually measure data heterogeneity.
Because the Dirichlet partitions are generated from fixed seeds, we can reconstruct each client's exact data and compute a ground-truth heterogeneity measure: the total-variation (TV) distance between the client's label distribution and the global label distribution.
Table~\ref{tab:rq2_correlations} correlates the final smoothed estimates $\tilde{H}_i$ with this ground truth, with the capacity constraints $p_i^{\max}$, and---via partial correlations---with the ground truth after linearly removing the capacity effect.

The answer is unambiguous, and negative.
The estimates correlate overwhelmingly with \emph{capacity} ($r$ between $-0.72$ and $-0.84$ on all six runs across both datasets), and once capacity is controlled for, no positive data signal remains: $r(\tilde{H}, \mathrm{TV} \mid p^{\max})$ ranges from $-0.60$ to $0.10$, so the residual association is either close to zero or of the opposite sign to the one the method assumes.
Consequently the allocations themselves do not track true heterogeneity: $r(p, \mathrm{TV} \mid p^{\max})$ ranges from $-0.46$ to $0.19$, lying within $0.1$ of zero on three runs and never reaching the clearly positive values the method requires.
A client that permanently trains a quarter-width sub-model produces updates whose shared coordinates compensate for its missing capacity, so its divergence from the aggregate reflects its \emph{width}, not its data; heavily skewed clients may even diverge less, since their few-class local tasks are easier and their gradients shrink faster.
We conclude that in system-heterogeneous federations, update divergence is confounded by capacity---a caution that applies to any method estimating client statistics from sub-model updates.
Whether the adaptive component contributes beyond its capacity budget is therefore tested directly by the random-budget control, which assigns each client a random time-varying width scaled to match HAS-FL's average capacity. On CIFAR-10 the two are indistinguishable within seed-to-seed variation: random-budget reaches $74.6\% \pm 5.1$ final accuracy against HAS-FL's $74.5\% \pm 5.6$ (mean per-seed difference $0.1$), while using slightly \emph{less} average capacity ($0.49$ vs.\ $0.52$). The result replicates on EMNIST ($77.4\% \pm 5.7$ for random-budget vs.\ $77.7\% \pm 4.6$ for HAS-FL, again at lower capacity), so it is not an artifact of a single task. Heterogeneity-aware allocation therefore adds nothing over random allocation at an equal budget on either dataset---the direct experimental counterpart to the confound: an estimator that cannot recover data heterogeneity cannot inform allocation, so acting on it is no better than chance.

\begin{table}[!htbp]
\centering
\caption{Correlations of the final estimates $\tilde{H}_i$ with ground-truth label divergence ($\mathrm{TV}$) and device capacity ($p^{\max}$).}
\label{tab:rq2_correlations}
\scriptsize
\setlength{\tabcolsep}{3pt}
\begin{tabular}{lcccc}
\toprule
Run & $r(\tilde{H},\mathrm{TV})$ & $r(\tilde{H},p^{\max})$ & $r(\tilde{H},\mathrm{TV}{\mid}p^{\max})$ & $r(p,\mathrm{TV}{\mid}p^{\max})$ \\
\midrule
CIFAR-10, s42 & $-0.44$ & $-0.72$ & $-0.60$ & $0.02$ \\
CIFAR-10, s43 & $-0.21$ & $-0.72$ & $-0.47$ & $-0.36$ \\
CIFAR-10, s44 & $0.21$ & $-0.84$ & $-0.40$ & $-0.46$ \\
EMNIST, s42 & $-0.43$ & $-0.75$ & $-0.26$ & $0.07$ \\
EMNIST, s43 & $-0.24$ & $-0.72$ & $-0.28$ & $0.05$ \\
EMNIST, s44 & $-0.02$ & $-0.81$ & $0.10$ & $0.19$ \\
\bottomrule
\end{tabular}
\end{table}

Figure~\ref{fig:rq2_dynamics} illustrates the temporal evolution of heterogeneity estimates and the resulting per-client allocations.
The left panels show the average heterogeneity estimate $\bar{H}$ exhibiting a characteristic sawtooth pattern on both datasets, rising gradually between normalization rounds before being reset to one.
The rescaling every 20 rounds keeps the absolute scale within a fixed range while preserving the relative ordering of client estimates that drives allocation decisions; the average oscillates between 1.0 and approximately 1.8--2.0 on both datasets.
The average allocation $\bar{p}$ remains stable near 0.5 throughout training. 

\begin{figure*}[!htbp]
    \centering
    \begin{minipage}{0.48\textwidth}
        \centering
        \begin{tikzpicture}
        \begin{axis}[
            width=0.9\textwidth,
            height=0.8\textwidth,
            xlabel={Communication Round},
            ylabel={Avg.\ heterogeneity $\bar{H}$},
            axis y line*=left,
            xmin=0, xmax=200,
            ymin=0.95, ymax=2.0,
            grid=both,
            grid style={line width=0.1pt, draw=gray!30},
            ylabel style={color=green!60!black},
            yticklabel style={color=green!60!black},
            label style={font=\small},
            tick label style={font=\footnotesize},
            ]
        \addplot[green!60!black, thick] coordinates {
    (10, 1.113) (15, 1.300) (20, 1.000) (25, 1.269) (30, 1.505) (35, 1.718) (40, 1.000) (45, 1.181) (50, 1.401) (55, 1.595) (60, 1.000) (65, 1.273) (70, 1.521) (75, 1.770) (80, 1.000) (85, 1.221) (90, 1.455) (95, 1.723) (100, 1.000) (105, 1.267) (110, 1.524) (115, 1.829) (120, 1.000) (125, 1.215) (130, 1.477) (135, 1.762) (140, 1.000) (145, 1.297) (150, 1.541) (155, 1.726) (160, 1.000) (165, 1.247) (170, 1.544) (175, 1.835) (180, 1.000) (185, 1.235) (190, 1.577) (195, 1.866)
};
        \end{axis}
        \begin{axis}[
            width=0.9\textwidth,
            height=0.8\textwidth,
            axis y line*=right,
            axis x line=none,
            xmin=0, xmax=200,
            ymin=0.49, ymax=0.56,
            ylabel style={color=purple},
            yticklabel style={color=purple},
            ytick={0.50, 0.52, 0.54, 0.56},
            tick label style={font=\footnotesize},
            legend style={font=\footnotesize, at={(0.5,1.03)}, anchor=south, legend columns=2, draw=none, fill=none, /tikz/every even column/.append style={column sep=8pt}},]
        \addplot[orange, thick, dotted] coordinates {(0, 0.525) (200, 0.525)};
        \addplot[purple, thick, dashed] coordinates {
    (10, 0.5366) (15, 0.5339) (20, 0.5314) (25, 0.5330) (30, 0.5279) (35, 0.5303) (40, 0.5286) (45, 0.5284) (50, 0.5252) (55, 0.5236) (60, 0.5233) (65, 0.5236) (70, 0.5221) (75, 0.5193) (80, 0.5194) (85, 0.5194) (90, 0.5224) (95, 0.5198) (100, 0.5216) (105, 0.5169) (110, 0.5169) (115, 0.5191) (120, 0.5192) (125, 0.5186) (130, 0.5189) (135, 0.5214) (140, 0.5209) (145, 0.5198) (150, 0.5158) (155, 0.5163) (160, 0.5174) (165, 0.5203) (170, 0.5162) (175, 0.5167) (180, 0.5194) (185, 0.5202) (190, 0.5198) (195, 0.5179)
};
        \legend{Mean cap $\bar{p}^{\max}$, HAS-FL $\bar{p}$}
        \end{axis}
        \end{tikzpicture}
        \centerline{(a) CIFAR-10}
    \end{minipage}
    \hfill
    \begin{minipage}{0.48\textwidth}
        \centering
        \begin{tikzpicture}
        \begin{axis}[
            width=0.9\textwidth,
            height=0.8\textwidth,
            xlabel={Communication Round},
            axis y line*=left,
            xmin=0, xmax=200,
            ymin=0.95, ymax=2.15,
            grid=both,
            grid style={line width=0.1pt, draw=gray!30},
            ylabel style={color=green!60!black},
            yticklabel style={color=green!60!black},
            label style={font=\small},
            tick label style={font=\footnotesize},
            ]
        \addplot[green!60!black, thick] coordinates {
    (10, 1.230) (15, 1.491) (20, 1.000) (25, 1.276) (30, 1.581) (35, 1.877) (40, 1.000) (45, 1.335) (50, 1.646) (55, 1.940) (60, 1.000) (65, 1.294) (70, 1.588) (75, 1.899) (80, 1.000) (85, 1.342) (90, 1.636) (95, 1.925) (100, 1.000) (105, 1.333) (110, 1.608) (115, 1.912) (120, 1.000) (125, 1.358) (130, 1.656) (135, 1.984) (140, 1.000) (145, 1.329) (150, 1.623) (155, 1.938) (160, 1.000) (165, 1.350) (170, 1.676) (175, 1.945) (180, 1.000) (185, 1.353) (190, 1.648) (195, 1.930)
};
        \end{axis}
        \begin{axis}[
            width=0.9\textwidth,
            height=0.8\textwidth,
            axis y line*=right,
            axis x line=none,
            xmin=0, xmax=200,
            ymin=0.49, ymax=0.56,
            ylabel={Avg.\ allocation $\bar{p}$},
            label style={font=\small},
            ylabel style={color=purple},
            yticklabel style={color=purple},
            tick label style={font=\footnotesize},
            legend style={font=\footnotesize, at={(0.5,1.03)}, anchor=south, legend columns=2, draw=none, fill=none, /tikz/every even column/.append style={column sep=8pt}},]
        \addplot[purple, thick, dashed] coordinates {
    (10, 0.5326) (15, 0.5299) (20, 0.5276) (25, 0.5302) (30, 0.5230) (35, 0.5259) (40, 0.5240) (45, 0.5234) (50, 0.5207) (55, 0.5193) (60, 0.5202) (65, 0.5167) (70, 0.5157) (75, 0.5145) (80, 0.5149) (85, 0.5171) (90, 0.5188) (95, 0.5173) (100, 0.5190) (105, 0.5134) (110, 0.5134) (115, 0.5177) (120, 0.5180) (125, 0.5200) (130, 0.5200) (135, 0.5210) (140, 0.5210) (145, 0.5197) (150, 0.5157) (155, 0.5158) (160, 0.5181) (165, 0.5226) (170, 0.5172) (175, 0.5184) (180, 0.5189) (185, 0.5225) (190, 0.5210) (195, 0.5188)
};
        \addplot[orange, dotted, very thick] coordinates {(0,0.525) (200,0.525)};
        \legend{HAS-FL $\bar{p}$, Mean cap $\bar{p}^{\max}$}
        \end{axis}
        \end{tikzpicture}
        \centerline{(b) EMNIST}
    \end{minipage}
    \caption{Average heterogeneity estimate and allocation during training on CIFAR-10 (a) and EMNIST (b).}
    \label{fig:rq2_dynamics}
\end{figure*}
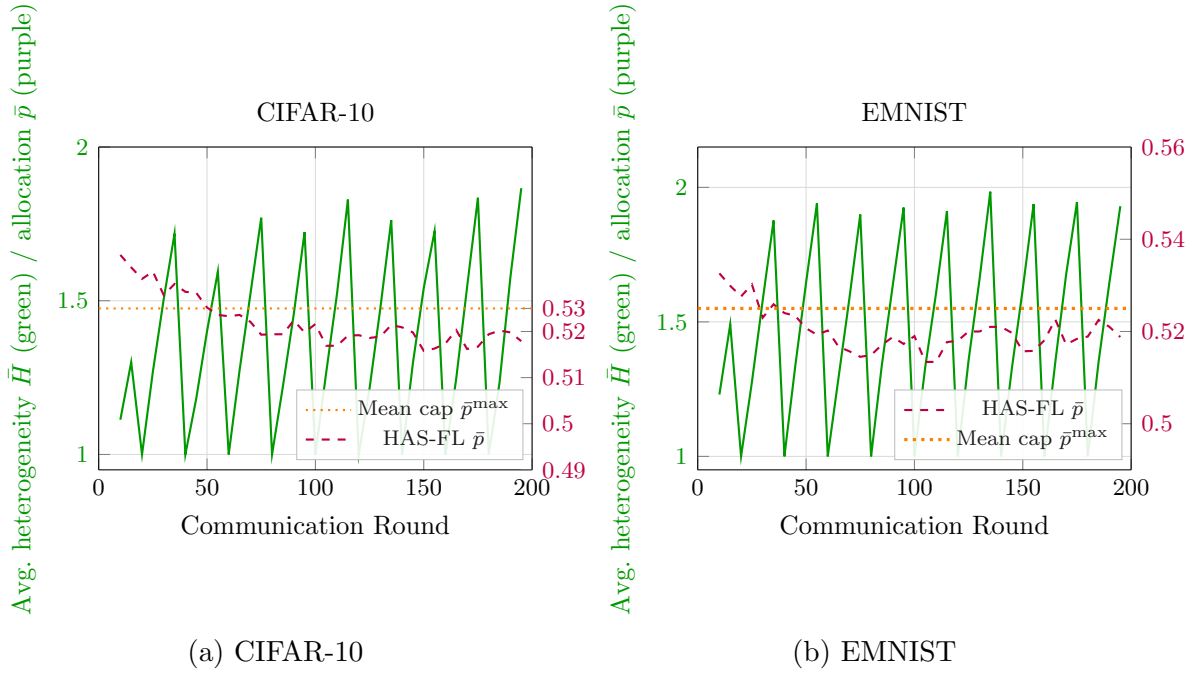

Figure~\ref{fig:rq2_allocations} shows the final per-client allocations. The allocations are visibly differentiated---but, consistent with Table~\ref{tab:rq2_correlations}, the differentiation follows the capacity tiers rather than the underlying data: low-capacity clients accumulate high divergence estimates and sit at their caps, while high-capacity clients appear homogeneous and receive reduced widths (subject to the coverage guarantee).

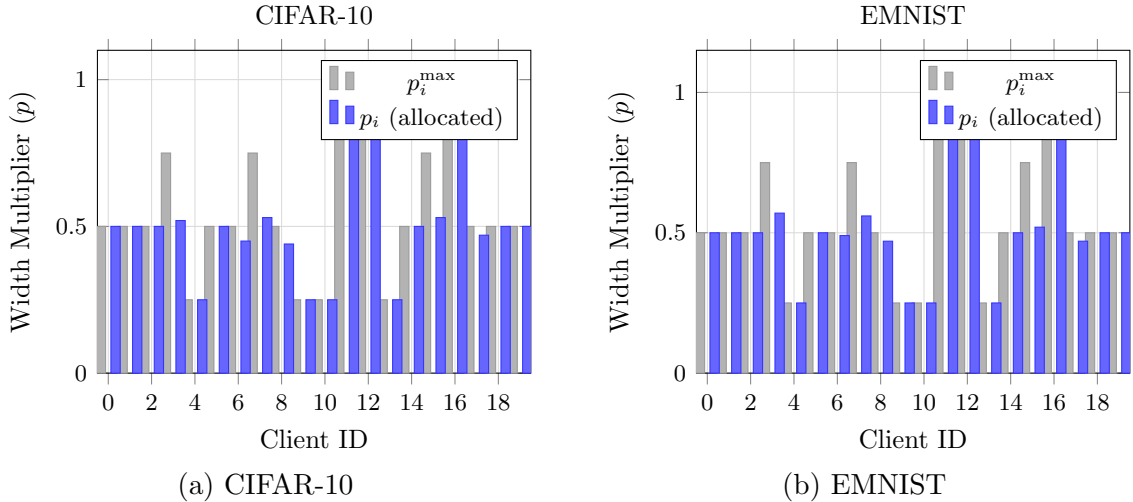
\begin{figure*}[!htbp]
    \centering
    \begin{minipage}{0.48\textwidth}
        \centering
        \begin{tikzpicture}
        \begin{axis}[
            width=\textwidth,
            height=0.8\textwidth,
            xlabel={Client ID},
            ylabel={Width Multiplier ($p$)},
            xmin=-0.5, xmax=19.5,
            ymin=0, ymax=1.1,
            xtick={0,2,4,6,8,10,12,14,16,18},
            grid=both,
            grid style={line width=0.1pt, draw=gray!30},
            ybar,
            bar width=3.5pt,
            legend style={font=\footnotesize, at={(0.5,1.03)}, anchor=south, legend columns=2, draw=none, fill=none, /tikz/every even column/.append style={column sep=8pt}},
            label style={font=\small},
            tick label style={font=\footnotesize},
        ]
        \addplot[fill=gray!60, draw=gray!80] coordinates {
    (0, 0.50) (1, 0.50) (2, 0.50) (3, 0.75) (4, 0.25) (5, 0.50) (6, 0.50) (7, 0.75) (8, 0.50) (9, 0.25) (10, 0.25) (11, 1.00) (12, 1.00) (13, 0.25) (14, 0.50) (15, 0.75) (16, 1.00) (17, 0.50) (18, 0.50) (19, 0.50)
};
        \addlegendentry{$p_i^{\max}$}
        \addplot[fill=blue!60, draw=blue!80] coordinates {
    (0, 0.50) (1, 0.50) (2, 0.50) (3, 0.52) (4, 0.25) (5, 0.50) (6, 0.45) (7, 0.53) (8, 0.44) (9, 0.25) (10, 0.25) (11, 1.00) (12, 1.00) (13, 0.25) (14, 0.50) (15, 0.53) (16, 1.00) (17, 0.47) (18, 0.50) (19, 0.50)
};
        \addlegendentry{$p_i$ (allocated)}
        \end{axis}
        \end{tikzpicture}
        \centerline{(a) CIFAR-10}
    \end{minipage}
    \hfill
    \begin{minipage}{0.48\textwidth}
        \centering
        \begin{tikzpicture}
        \begin{axis}[
            width=\textwidth,
            height=0.8\textwidth,
            xlabel={Client ID},
            ylabel={Width Multiplier ($p$)},
            xmin=-0.5, xmax=19.5,
            ymin=0, ymax=1.15,
            xtick={0,2,4,6,8,10,12,14,16,18},
            grid=both,
            grid style={line width=0.1pt, draw=gray!30},
            ybar,
            bar width=3.5pt,
            legend style={font=\footnotesize, at={(0.5,1.03)}, anchor=south, legend columns=2, draw=none, fill=none, /tikz/every even column/.append style={column sep=8pt}},
            label style={font=\small},
            tick label style={font=\footnotesize},
        ]
        \addplot[fill=gray!60, draw=gray!80] coordinates {
    (0, 0.50) (1, 0.50) (2, 0.50) (3, 0.75) (4, 0.25) (5, 0.50) (6, 0.50) (7, 0.75) (8, 0.50) (9, 0.25) (10, 0.25) (11, 1.00) (12, 1.00) (13, 0.25) (14, 0.50) (15, 0.75) (16, 1.00) (17, 0.50) (18, 0.50) (19, 0.50)
};
        \addlegendentry{$p_i^{\max}$}
        \addplot[fill=blue!60, draw=blue!80] coordinates {
    (0, 0.50) (1, 0.50) (2, 0.50) (3, 0.57) (4, 0.25) (5, 0.50) (6, 0.49) (7, 0.56) (8, 0.47) (9, 0.25) (10, 0.25) (11, 1.00) (12, 1.00) (13, 0.25) (14, 0.50) (15, 0.52) (16, 1.00) (17, 0.47) (18, 0.50) (19, 0.50)
};
        \addlegendentry{$p_i$ (allocated)}
        \end{axis}
        \end{tikzpicture}
        \centerline{(b) EMNIST}
    \end{minipage}
    \caption{Final per-client allocations $p_i$ (blue) vs.\ resource constraints $p_i^{\max}$ (gray).}
    \label{fig:rq2_allocations}
\end{figure*}

The practical takeaway is that adaptive allocation operates at a lower average capacity than static resource-only assignment while matching its accuracy, but the analysis above attributes this to disciplined capacity budgeting under a coverage guarantee rather than to successful heterogeneity targeting.

\subsection{Comparison with Full-Model Methods}
\label{subsec:rq3}

Our third research question examines how sub-model training, with HAS-FL as its adaptive instance, compares against full-model federated optimization methods designed specifically for statistical heterogeneity.
We compare against FedAvg as the standard baseline, SCAFFOLD which employs control variates to correct client drift, and FedProx which adds proximal regularization to constrain local updates.
This comparison places sub-model training relative to full-model approaches that address data heterogeneity through optimization techniques rather than capacity allocation.
Table~\ref{tab:rq3_comparison} places the sub-model methods against full-model training, and the gap is substantial. FedAvg reaches $88.2\% \pm 0.6$ final accuracy on CIFAR-10. Training at roughly half capacity therefore carries a real, quantifiable cost---HAS-FL ($74.5\% \pm 5.6$) gives up about 13 percentage points relative to that ceiling in exchange for roughly a threefold reduction in average per-client computation and upload volume. A client at exactly half width costs about a quarter of a full one (Figure~\ref{fig:cost_model}), but cost is quadratic in width, so the average over the allocated widths is higher. The random-budget control lands exactly on HAS-FL ($74.6\% \pm 5.1$ at slightly lower capacity), confirming at the system level that the adaptive component contributes nothing beyond its capacity budget. FedProx with $\mu = 1.0$ ends at $65.5\% \pm 2.0$ on CIFAR-10 but is healthy on EMNIST ($83.9\% \pm 1.3$): the proximal term throttles local progress precisely under the severe ten-class label skew, and is comparatively benign on the milder 62-class partition---itself a small heterogeneity-sensitivity signal. On EMNIST the full-model ceiling is closer: SCAFFOLD ($87.2\% \pm 0.3$) and FedAvg ($85.5\% \pm 1.7$) lead HAS-FL ($77.7\%$) by roughly eight points rather than the thirteen seen on CIFAR-10, so the capacity--accuracy trade-off softens as the task grows. SCAFFOLD is the strongest full-model baseline on both datasets once its control variates are run under vanilla local SGD ($86.3\% \pm 1.2$ on CIFAR-10, $87.2\% \pm 0.3$ on EMNIST); under the momentum-based optimizer shared by the other methods it diverges to chance accuracy, so we report it in its canonical configuration and note the incompatibility explicitly.

\begin{table*}[!htbp]
\centering
\caption{Comparison with full-model methods on CIFAR-10 and EMNIST.}
\label{tab:rq3_comparison}
\small\setlength{\tabcolsep}{4pt}
\begin{tabular}{llcccc}
\toprule
Dataset & Algorithm & Type & Best Acc. & Final Acc. & Avg. $p$ \\
\midrule
\multirow{5}{*}{CIFAR-10}
& FedAvg & Full-model & 88.5 $\pm$ 0.5\% & 88.2 $\pm$ 0.6\% & 1.00 \\
& SCAFFOLD & Full-model + CV & 86.7 $\pm$ 0.9\% & 86.3 $\pm$ 1.2\% & 1.00 \\
& FedProx & Full-model + prox & 66.4 $\pm$ 1.0\% & 65.5 $\pm$ 2.0\% & 1.00 \\
& HAS-FL & Sub-model (adaptive) & 76.7 $\pm$ 4.9\% & 74.5 $\pm$ 5.6\% & 0.52 \\
& Random-budget & Sub-model (random) & 75.9 $\pm$ 4.9\% & 74.6 $\pm$ 5.1\% & 0.49 \\
\midrule
\multirow{5}{*}{EMNIST}
& FedAvg & Full-model & 87.4 $\pm$ 0.3\% & 85.5 $\pm$ 1.7\% & 1.00 \\
& SCAFFOLD & Full-model + CV & 87.5 $\pm$ 0.2\% & 87.2 $\pm$ 0.3\% & 1.00 \\
& FedProx & Full-model + prox & 85.0 $\pm$ 0.0\% & 83.9 $\pm$ 1.3\% & 1.00 \\
& HAS-FL & Sub-model (adaptive) & 82.0 $\pm$ 1.8\% & 77.7 $\pm$ 4.6\% & 0.52 \\
& Random-budget & Sub-model (random) & 82.4 $\pm$ 2.1\% & 77.4 $\pm$ 5.7\% & 0.49 \\
\bottomrule
\end{tabular}
\end{table*}

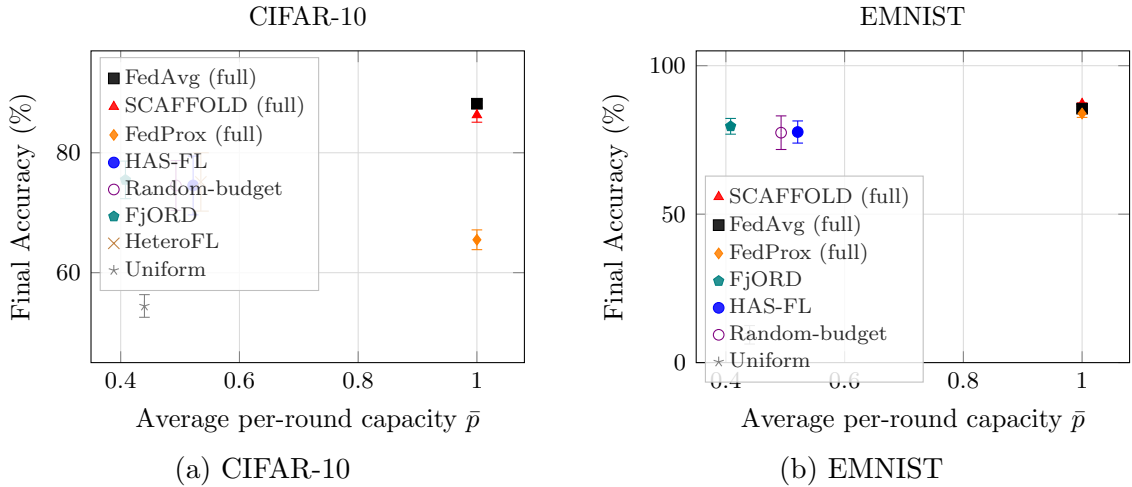
\begin{figure*}[!htbp]
\centering
\begin{tikzpicture}
\begin{groupplot}[
    group style={group size=2 by 1, horizontal sep=14mm},
    width=0.53\textwidth, height=6.2cm,
    ylabel={Final accuracy (\%)},
    xlabel={Average per-round capacity $\bar{p}$},
    xmin=0.33, xmax=1.08,
    xtick={0.4,0.5,0.6,0.7,0.8,0.9,1.0},
    grid=major, grid style={draw=gray!20}, clip=false,
    label style={font=\footnotesize}, tick label style={font=\footnotesize},
    title style={font=\footnotesize\sffamily\bfseries},
    legend columns=4,
    legend cell align=left,
    legend style={font=\scriptsize, draw=none, at={(1.08,1.20)}, anchor=south, /tikz/every even column/.append style={column sep=3mm}},
]
\nextgroupplot[title={(a) CIFAR-10}, ymin=45, ymax=97]
\fill[gray!12] (axis cs:0.37,45) rectangle (axis cs:0.57,97);
\node[font=\scriptsize\sffamily,gray!50!black,anchor=north] at (axis cs:0.47,96) {sub-model};
\fill[gray!12] (axis cs:0.93,45) rectangle (axis cs:1.07,97);
\node[font=\scriptsize\sffamily,gray!50!black,anchor=north] at (axis cs:1.0,96) {full};
\addplot[only marks,mark=square*,mark size=2.6pt,black,error bars/.cd,y dir=both,y explicit,error bar style={thin}] coordinates {(1.0,88.18) +- (0,0.51)};\addlegendentry{FedAvg (full)}
\addplot[only marks,mark=triangle*,mark size=3.2pt,red!80!black,error bars/.cd,y dir=both,y explicit,error bar style={thin}] coordinates {(1.0,86.3) +- (0,1.2)};\addlegendentry{SCAFFOLD (full)}
\addplot[only marks,mark=diamond*,mark size=3.2pt,orange!90!black,error bars/.cd,y dir=both,y explicit,error bar style={thin}] coordinates {(1.0,65.49) +- (0,1.65)};\addlegendentry{FedProx (full)}
\addplot[only marks,mark=pentagon*,mark size=2.8pt,teal,error bars/.cd,y dir=both,y explicit,error bar style={thin}] coordinates {(0.408,75.48) +- (0,3.13)};\addlegendentry{FjORD}
\addplot[only marks,mark=*,mark size=2.6pt,blue,error bars/.cd,y dir=both,y explicit,error bar style={thin}] coordinates {(0.522,74.55) +- (0,4.9)};\addlegendentry{HAS-FL}
\addplot[only marks,mark=o,mark size=2.8pt,violet,thick,error bars/.cd,y dir=both,y explicit,error bar style={thin}] coordinates {(0.493,74.55) +- (0,4.16)};\addlegendentry{Random-budget}
\addplot[only marks,mark=x,mark size=3.4pt,brown,thick,error bars/.cd,y dir=both,y explicit,error bar style={thin}] coordinates {(0.535,75.14) +- (0,4.87)};\addlegendentry{HeteroFL}
\addplot[only marks,mark=star,mark size=3.4pt,gray!70!black,thick,error bars/.cd,y dir=both,y explicit,error bar style={thin}] coordinates {(0.44,54.45) +- (0,1.89)};\addlegendentry{Uniform}
\node[font=\scriptsize,anchor=west] at (axis cs:0.455,54.45) {no coverage};
\nextgroupplot[title={(b) EMNIST}, ymin=0, ymax=105]
\fill[gray!12] (axis cs:0.37,0) rectangle (axis cs:0.57,105);
\node[font=\scriptsize\sffamily,gray!50!black,anchor=north] at (axis cs:0.47,104) {sub-model};
\fill[gray!12] (axis cs:0.93,0) rectangle (axis cs:1.07,105);
\node[font=\scriptsize\sffamily,gray!50!black,anchor=north] at (axis cs:1.0,104) {full};
\addplot[only marks,mark=square*,mark size=2.6pt,black,error bars/.cd,y dir=both,y explicit,error bar style={thin}] coordinates {(1.0,85.53) +- (0,1.4)};
\addplot[only marks,mark=triangle*,mark size=3.2pt,red!80!black,error bars/.cd,y dir=both,y explicit,error bar style={thin}] coordinates {(1.0,87.2) +- (0,0.3)};
\addplot[only marks,mark=diamond*,mark size=3.2pt,orange!90!black,error bars/.cd,y dir=both,y explicit,error bar style={thin}] coordinates {(1.0,83.9) +- (0,1.3)};
\addplot[only marks,mark=pentagon*,mark size=2.8pt,teal,error bars/.cd,y dir=both,y explicit,error bar style={thin}] coordinates {(0.408,79.59) +- (0,2.65)};
\addplot[only marks,mark=*,mark size=2.6pt,blue,error bars/.cd,y dir=both,y explicit,error bar style={thin}] coordinates {(0.521,77.67) +- (0,3.73)};
\addplot[only marks,mark=o,mark size=2.8pt,violet,thick,error bars/.cd,y dir=both,y explicit,error bar style={thin}] coordinates {(0.493,77.44) +- (0,5.66)};
\addplot[only marks,mark=star,mark size=3.4pt,gray!70!black,thick,error bars/.cd,y dir=both,y explicit,error bar style={thin}] coordinates {(0.44,9.37) +- (0,3.12)};
\node[font=\scriptsize,anchor=west] at (axis cs:0.455,9.37) {no coverage};
\end{groupplot}
\end{tikzpicture}
\caption{Final accuracy against average per-round capacity (mean $\pm$ std over three seeds) on (a) CIFAR-10 and (b) EMNIST. The sub-model methods that preserve parameter coverage (HAS-FL, Random-budget, FjORD and HeteroFL) reach the same accuracy within seed-to-seed variation although their capacities range from 0.41 to 0.54, so heterogeneity-aware allocation brings no gain over a random policy at a matched budget. Full-model FedAvg and SCAFFOLD remain eight to thirteen points higher at full cost, and Uniform, which leaves parameters untrained, falls far below.}
\label{fig:pareto}
\end{figure*}

Figure~\ref{fig:rq3_convergence} shows the convergence trajectories. On CIFAR-10, FedAvg's full-model curve rises smoothly and saturates near 88\%, visibly separated from the sub-model band throughout training; HAS-FL matches its early pace up to roughly round 15 before flattening near 75\%, and FedProx climbs slowly but monotonically toward 65\%, its proximal term damping both oscillation and progress. On EMNIST every method exceeds 80\% within five rounds and the separation between full- and half-capacity training narrows to roughly eight points, with FedProx recovering to within two points of FedAvg by the end of training.

\begin{figure*}[!htbp]
    \centering
    \begin{minipage}{0.48\textwidth}
        \centering
        \begin{tikzpicture}
        \begin{axis}[
            width=\textwidth,
            height=0.8\textwidth,
            xlabel={Round},
            ylabel={Test Accuracy (\%)},
            xmin=0, xmax=200,
            ymin=0, ymax=90,
            grid=both,
            grid style={line width=0.1pt, draw=gray!30},
            major grid style={line width=0.2pt, draw=gray!50},
            legend pos=south east,
            legend style={font=\scriptsize, cells={anchor=west}},
            legend columns=1,
        ]
        \addplot[blue, thick] coordinates {
    (0, 17.95) (5, 38.72) (10, 52.46) (15, 61.03) (20, 57.85) (25, 63.64) (30, 65.11) (35, 63.62) (40, 67.62) (45, 66.77) (50, 67.89) (55, 68.14) (60, 70.21) (65, 68.30) (70, 72.14) (75, 70.19) (80, 71.90) (85, 70.55) (90, 71.16) (95, 71.09) (100, 73.64) (105, 73.16) (110, 71.04) (115, 73.79) (120, 75.16) (125, 73.10) (130, 72.96) (135, 69.80) (140, 73.10) (145, 70.52) (150, 72.59) (155, 73.89) (160, 72.42) (165, 75.54) (170, 72.74) (175, 74.65) (180, 75.48) (185, 71.38) (190, 72.33) (195, 74.45) (199, 74.52)
};
        \addlegendentry{HAS-FL}
        
        \addplot[black, thick, dashed] coordinates {
    (0, 10.01) (5, 39.78) (10, 55.95) (15, 66.16) (20, 70.54) (25, 74.20) (30, 77.57) (35, 79.39) (40, 80.58) (45, 81.43) (50, 82.39) (55, 82.82) (60, 83.81) (65, 85.10) (70, 85.24) (75, 85.40) (80, 85.56) (85, 85.93) (90, 85.86) (95, 85.97) (100, 86.56) (105, 87.06) (110, 86.85) (115, 87.26) (120, 86.61) (125, 87.16) (130, 87.48) (135, 87.35) (140, 87.18) (145, 87.65) (150, 87.53) (155, 88.05) (160, 87.73) (165, 88.29) (170, 87.96) (175, 87.56) (180, 88.13) (185, 87.65) (190, 87.85) (195, 88.17) (199, 88.18)
};
        \addlegendentry{FedAvg}
        
        \addplot[red, thick, densely dashed] coordinates {
    (0, 26.92) (5, 50.94) (10, 58.89) (15, 65.40) (20, 69.57) (25, 72.01) (30, 74.25) (35, 75.71) (40, 76.52) (45, 77.74) (50, 77.87) (55, 79.02) (60, 80.03) (65, 80.88) (70, 81.16) (75, 81.69) (80, 81.91) (85, 82.50) (90, 82.52) (95, 83.10) (100, 83.13) (105, 83.97) (110, 84.15) (115, 84.11) (120, 84.21) (125, 84.66) (130, 84.90) (135, 84.78) (140, 85.19) (145, 85.16) (150, 85.34) (155, 85.67) (160, 85.64) (165, 85.95) (170, 85.70) (175, 86.02) (180, 86.39) (185, 86.02) (190, 86.07) (195, 86.42) (199, 86.26)
};
        \addlegendentry{SCAFFOLD}
                
        \addplot[orange, thick, loosely dashed] coordinates {
    (0, 13.95) (5, 26.09) (10, 30.63) (15, 33.77) (20, 34.99) (25, 37.03) (30, 38.85) (35, 39.75) (40, 39.80) (45, 40.82) (50, 41.89) (55, 43.45) (60, 44.19) (65, 46.95) (70, 47.50) (75, 48.06) (80, 48.98) (85, 49.61) (90, 50.70) (95, 51.68) (100, 52.66) (105, 53.59) (110, 54.70) (115, 55.49) (120, 53.36) (125, 56.12) (130, 58.34) (135, 56.87) (140, 57.74) (145, 59.61) (150, 58.74) (155, 61.87) (160, 61.78) (165, 61.97) (170, 62.83) (175, 62.35) (180, 63.84) (185, 65.04) (190, 65.81) (195, 66.00) (199, 65.49)
};
        \addlegendentry{FedProx}
        \end{axis}
        \end{tikzpicture}
        \centerline{(a) CIFAR-10}
    \end{minipage}
    \hfill
    \begin{minipage}{0.48\textwidth}
        \centering
        \begin{tikzpicture}
        \begin{axis}[
            width=\textwidth,
            height=0.8\textwidth,
            xlabel={Round},
            ylabel={Test Accuracy (\%)},
            xmin=0, xmax=200,
            ymin=0, ymax=95,
            grid=both,
            grid style={line width=0.1pt, draw=gray!30},
            major grid style={line width=0.2pt, draw=gray!50},
            legend pos=south east,
            legend style={font=\scriptsize, cells={anchor=west}},
            legend columns=1,
        ]
        \addplot[blue, thick] coordinates {
    (0, 25.39) (5, 79.94) (10, 81.50) (15, 80.31) (20, 80.65) (25, 79.94) (30, 80.03) (35, 77.92) (40, 80.75) (45, 78.49) (50, 79.62) (55, 79.07) (60, 78.36) (65, 80.32) (70, 78.35) (75, 78.49) (80, 80.27) (85, 79.19) (90, 77.44) (95, 78.08) (100, 78.54) (105, 78.56) (110, 77.89) (115, 77.38) (120, 78.42) (125, 77.62) (130, 78.84) (135, 76.96) (140, 75.90) (145, 76.92) (150, 78.34) (155, 78.26) (160, 79.25) (165, 80.38) (170, 78.46) (175, 80.30) (180, 78.86) (185, 74.75) (190, 77.63) (195, 78.14) (199, 77.67)
};
        
        \addplot[red, thick, densely dashed] coordinates {
    (0, 57.34) (5, 82.49) (10, 84.76) (15, 85.44) (20, 85.83) (25, 85.64) (30, 86.28) (35, 86.20) (40, 86.37) (45, 86.25) (50, 86.39) (55, 86.57) (60, 86.75) (65, 86.70) (70, 86.67) (75, 86.75) (80, 86.45) (85, 86.84) (90, 86.90) (95, 86.98) (100, 86.73) (105, 86.88) (110, 86.99) (115, 86.88) (120, 87.04) (125, 87.20) (130, 87.10) (135, 87.02) (140, 86.93) (145, 87.17) (150, 87.00) (155, 87.24) (160, 87.09) (165, 87.14) (170, 87.09) (175, 87.07) (180, 87.20) (185, 87.35) (190, 87.35) (195, 87.17) (199, 87.23)
};
                
        \addplot[orange, thick, loosely dashed] coordinates {
    (0, 22.26) (5, 59.39) (10, 68.33) (15, 72.18) (20, 74.10) (25, 74.94) (30, 76.83) (35, 77.70) (40, 79.62) (45, 79.11) (50, 79.85) (55, 80.46) (60, 81.58) (65, 80.60) (70, 81.23) (75, 81.87) (80, 81.53) (85, 81.56) (90, 82.05) (95, 82.99) (100, 83.11) (105, 82.07) (110, 83.65) (115, 82.86) (120, 83.94) (125, 84.20) (130, 83.77) (135, 83.22) (140, 83.69) (145, 84.23) (150, 84.24) (155, 84.14) (160, 84.33) (165, 83.78) (170, 84.50) (175, 84.00) (180, 84.48) (185, 84.18) (190, 84.60) (195, 84.35) (199, 83.93)
};
        
        \addplot[black, thick, dashed] coordinates {
    (0, 5.74) (5, 80.75) (10, 84.68) (15, 83.99) (20, 85.09) (25, 84.64) (30, 85.34) (35, 85.58) (40, 85.34) (45, 85.28) (50, 85.31) (55, 85.33) (60, 86.28) (65, 85.01) (70, 84.64) (75, 86.14) (80, 85.31) (85, 85.16) (90, 84.70) (95, 86.32) (100, 86.27) (105, 85.75) (110, 86.61) (115, 86.32) (120, 86.40) (125, 87.15) (130, 86.63) (135, 85.89) (140, 85.87) (145, 86.71) (150, 86.57) (155, 86.70) (160, 86.74) (165, 85.18) (170, 86.79) (175, 85.98) (180, 86.71) (185, 86.12) (190, 86.62) (195, 86.54) (199, 85.52)
};
        \end{axis}
        \end{tikzpicture}
        \centerline{(b) EMNIST}
    \end{minipage}
    \caption{Convergence comparison with full-model methods on CIFAR-10 (a) and EMNIST (b).}
    \label{fig:rq3_convergence}
\end{figure*}
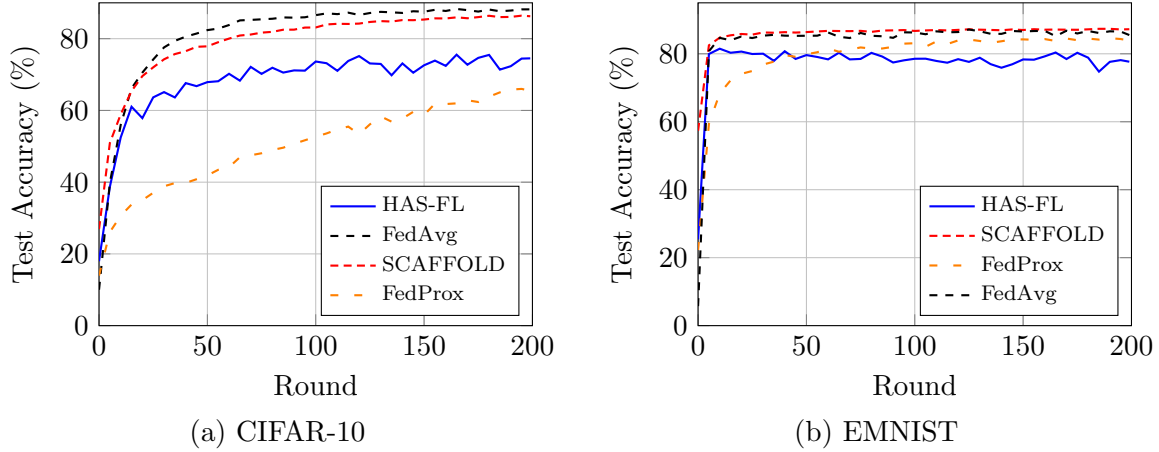

The loss trajectories in Figure~\ref{fig:rq3_loss} add a calibration perspective. FedAvg attains and holds the lowest loss on both datasets (about 0.45 on CIFAR-10 and 0.35 on EMNIST). FedProx's loss decreases monotonically to the very end of training on both datasets, consistent with an under-converged but stable proximal trajectory. HAS-FL reaches its loss minimum early (about 1.24 on CIFAR-10 and 0.50 on EMNIST) and then drifts upward even as its accuracy stays flat---a mild miscalibration that coordinate-wise averaging of heterogeneous sub-models does not remove.

\begin{figure*}[!htbp]
    \centering
    \begin{minipage}{0.48\textwidth}
        \centering
        \begin{tikzpicture}
        \begin{axis}[
            width=\textwidth,
            height=0.8\textwidth,
            xlabel={Round},
            ylabel={Test Loss},
            xmin=0, xmax=200,
            ymin=0.3, ymax=2.5,
            grid=both,
            grid style={line width=0.1pt, draw=gray!30},
            major grid style={line width=0.2pt, draw=gray!50},
            legend pos=north east,
            legend style={font=\scriptsize, cells={anchor=west}},
            legend columns=1,
        ]
        \addplot[blue, thick] coordinates {
    (0, 2.32) (5, 1.71) (10, 1.37) (15, 1.18) (20, 1.55) (25, 1.34) (30, 1.26) (35, 1.43) (40, 1.25) (45, 1.35) (50, 1.34) (55, 1.52) (60, 1.42) (65, 1.52) (70, 1.28) (75, 1.39) (80, 1.38) (85, 1.51) (90, 1.52) (95, 1.63) (100, 1.31) (105, 1.43) (110, 1.70) (115, 1.40) (120, 1.38) (125, 1.53) (130, 1.54) (135, 1.79) (140, 1.54) (145, 1.78) (150, 1.68) (155, 1.48) (160, 1.73) (165, 1.41) (170, 1.75) (175, 1.58) (180, 1.53) (185, 1.83) (190, 1.76) (195, 1.61) (199, 1.56)
};
        
        \addplot[black, thick, dashed] coordinates {
    (0, 2.30) (5, 1.59) (10, 1.20) (15, 0.94) (20, 0.84) (25, 0.74) (30, 0.65) (35, 0.61) (40, 0.57) (45, 0.55) (50, 0.54) (55, 0.53) (60, 0.50) (65, 0.46) (70, 0.47) (75, 0.46) (80, 0.48) (85, 0.46) (90, 0.47) (95, 0.47) (100, 0.45) (105, 0.44) (110, 0.45) (115, 0.44) (120, 0.47) (125, 0.45) (130, 0.44) (135, 0.46) (140, 0.46) (145, 0.44) (150, 0.45) (155, 0.43) (160, 0.46) (165, 0.44) (170, 0.46) (175, 0.47) (180, 0.44) (185, 0.48) (190, 0.46) (195, 0.45) (199, 0.47)
};
        
        \addplot[red, thick, densely dashed] coordinates {
    (0, 2.08) (5, 1.35) (10, 1.14) (15, 0.97) (20, 0.87) (25, 0.80) (30, 0.73) (35, 0.69) (40, 0.68) (45, 0.64) (50, 0.63) (55, 0.60) (60, 0.58) (65, 0.55) (70, 0.55) (75, 0.53) (80, 0.53) (85, 0.52) (90, 0.51) (95, 0.50) (100, 0.49) (105, 0.47) (110, 0.48) (115, 0.47) (120, 0.47) (125, 0.46) (130, 0.46) (135, 0.46) (140, 0.45) (145, 0.45) (150, 0.44) (155, 0.44) (160, 0.45) (165, 0.44) (170, 0.44) (175, 0.44) (180, 0.42) (185, 0.45) (190, 0.44) (195, 0.43) (199, 0.43)
};
                
        \addplot[orange, thick, loosely dashed] coordinates {
    (0, 2.28) (5, 2.01) (10, 1.90) (15, 1.81) (20, 1.82) (25, 1.74) (30, 1.69) (35, 1.67) (40, 1.68) (45, 1.61) (50, 1.58) (55, 1.56) (60, 1.53) (65, 1.46) (70, 1.43) (75, 1.42) (80, 1.40) (85, 1.38) (90, 1.34) (95, 1.33) (100, 1.30) (105, 1.28) (110, 1.24) (115, 1.23) (120, 1.27) (125, 1.20) (130, 1.15) (135, 1.18) (140, 1.16) (145, 1.11) (150, 1.13) (155, 1.06) (160, 1.07) (165, 1.06) (170, 1.05) (175, 1.05) (180, 1.00) (185, 0.99) (190, 0.96) (195, 0.94) (199, 0.98)
};
        \end{axis}
        \end{tikzpicture}
        \centerline{(a) CIFAR-10}
    \end{minipage}
    \hfill
    \begin{minipage}{0.48\textwidth}
        \centering
        \begin{tikzpicture}
        \begin{axis}[
            width=\textwidth,
            height=0.8\textwidth,
            xlabel={Round},
            ylabel={Test Loss},
            xmin=0, xmax=200,
            ymin=0, ymax=4.2,
            grid=both,
            grid style={line width=0.1pt, draw=gray!30},
            major grid style={line width=0.2pt, draw=gray!50},
            legend pos=north east,
            legend style={font=\scriptsize, cells={anchor=west}},
            legend columns=1,
        ]
        \addplot[blue, thick] coordinates {
    (0, 3.59) (5, 0.59) (10, 0.50) (15, 0.57) (20, 0.61) (25, 0.67) (30, 0.65) (35, 0.67) (40, 0.58) (45, 0.63) (50, 0.58) (55, 0.76) (60, 0.66) (65, 0.62) (70, 0.73) (75, 0.82) (80, 0.60) (85, 0.70) (90, 0.75) (95, 0.86) (100, 0.75) (105, 0.73) (110, 0.78) (115, 0.82) (120, 0.77) (125, 0.77) (130, 0.69) (135, 0.86) (140, 0.98) (145, 0.93) (150, 0.89) (155, 0.76) (160, 0.83) (165, 0.73) (170, 0.82) (175, 0.73) (180, 0.84) (185, 1.22) (190, 1.02) (195, 0.88) (199, 0.95)
};
        
        \addplot[red, thick, densely dashed] coordinates {
    (0, 1.88) (5, 0.47) (10, 0.41) (15, 0.39) (20, 0.38) (25, 0.37) (30, 0.36) (35, 0.36) (40, 0.36) (45, 0.36) (50, 0.35) (55, 0.35) (60, 0.35) (65, 0.35) (70, 0.35) (75, 0.35) (80, 0.35) (85, 0.34) (90, 0.34) (95, 0.34) (100, 0.34) (105, 0.34) (110, 0.34) (115, 0.34) (120, 0.34) (125, 0.34) (130, 0.34) (135, 0.34) (140, 0.34) (145, 0.34) (150, 0.34) (155, 0.33) (160, 0.34) (165, 0.34) (170, 0.34) (175, 0.34) (180, 0.34) (185, 0.33) (190, 0.33) (195, 0.33) (199, 0.33)
};
                
        \addplot[orange, thick, loosely dashed] coordinates {
    (0, 3.72) (5, 1.44) (10, 1.02) (15, 0.86) (20, 0.78) (25, 0.73) (30, 0.68) (35, 0.63) (40, 0.60) (45, 0.59) (50, 0.57) (55, 0.56) (60, 0.53) (65, 0.54) (70, 0.52) (75, 0.50) (80, 0.49) (85, 0.49) (90, 0.48) (95, 0.47) (100, 0.46) (105, 0.47) (110, 0.45) (115, 0.45) (120, 0.44) (125, 0.44) (130, 0.44) (135, 0.45) (140, 0.44) (145, 0.43) (150, 0.43) (155, 0.43) (160, 0.42) (165, 0.43) (170, 0.42) (175, 0.42) (180, 0.42) (185, 0.42) (190, 0.41) (195, 0.41) (199, 0.42)
};
        
        \addplot[black, thick, dashed] coordinates {
    (0, 4.11) (5, 0.51) (10, 0.41) (15, 0.41) (20, 0.39) (25, 0.40) (30, 0.38) (35, 0.37) (40, 0.38) (45, 0.37) (50, 0.37) (55, 0.37) (60, 0.36) (65, 0.38) (70, 0.38) (75, 0.35) (80, 0.37) (85, 0.37) (90, 0.37) (95, 0.36) (100, 0.35) (105, 0.35) (110, 0.35) (115, 0.35) (120, 0.35) (125, 0.34) (130, 0.35) (135, 0.36) (140, 0.36) (145, 0.35) (150, 0.35) (155, 0.35) (160, 0.34) (165, 0.37) (170, 0.34) (175, 0.35) (180, 0.35) (185, 0.35) (190, 0.35) (195, 0.34) (199, 0.37)
};
        \end{axis}
        \end{tikzpicture}
        \centerline{(b) EMNIST}
    \end{minipage}
    \caption{Test-loss trajectories on CIFAR-10 (a) and EMNIST (b). Line styles follow the legend of Figure~\ref{fig:rq3_convergence}(a).}
    \label{fig:rq3_loss}
\end{figure*}
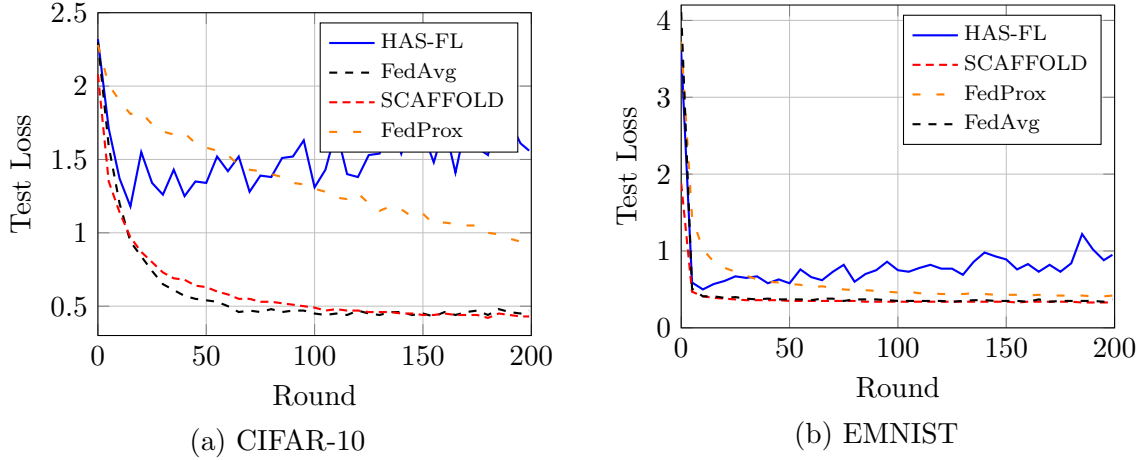

Taken together, the comparison reframes what sub-model training buys. It does not match full-model accuracy---the full-model baselines show a consistent double-digit gap to FedAvg on CIFAR-10---but it purchases the participation of resource-constrained devices at a quadratically reduced compute and communication cost. Within the sub-model regime, the decisive design choice is parameter coverage rather than allocation intelligence: HAS-FL, FjORD, and the random-budget control differ by less than their seed-to-seed variation even though their average capacities range from $0.41$ to $0.52$, so accuracy is insensitive to budget across this range, while the coverage-violating Uniform policy fails catastrophically. The cost of reduced capacity shows up instead as the gap to full-model training. Methods addressing statistical heterogeneity through full-model optimization and methods addressing system heterogeneity through sub-model training therefore occupy different points on an accuracy--cost frontier (Figure~\ref{fig:pareto}), and combining gradient correction with coverage-preserving sub-model training remains an open direction.

\subsection{Natural Partitions: Shakespeare}
\label{subsec:shakespeare}

The Dirichlet partitions used above impose heterogeneity by construction. We therefore repeat the allocation comparison on the naturally partitioned Shakespeare benchmark, where each client is one speaking role and heterogeneity arises from genuine stylistic differences. Table~\ref{tab:shakespeare} reports next-character prediction accuracy over three seeds.

\begin{table}[!htbp]
\centering
\caption{Allocation strategies on the naturally partitioned Shakespeare benchmark (mean $\pm$ std over three seeds).}
\label{tab:shakespeare}
\footnotesize
\begin{tabular}{lccc}
\toprule
Algorithm & Best Acc. & Final Acc. & Avg. $p$ \\
\midrule
HAS-FL & 29.5 $\pm$ 2.2\% & 25.5 $\pm$ 0.8\% & 0.53 \\
FjORD & 33.5 $\pm$ 0.9\% & 32.6 $\pm$ 1.3\% & 0.41 \\
Uniform & \textbf{35.0 $\pm$ 0.5}\% & \textbf{34.4 $\pm$ 0.6}\% & 0.44 \\
\bottomrule
\end{tabular}
\end{table}

The ordering inverts relative to the image benchmarks, and does so in the direction least favourable to adaptive allocation. HAS-FL is the \emph{weakest} of the three policies despite consuming the \emph{largest} average capacity, and it is the only one whose accuracy degrades between its peak and the end of training (a four-point decline, against one point for FjORD and half a point for Uniform). On the one benchmark whose heterogeneity is not synthetic, allocating capacity by estimated divergence is therefore not merely uninformative---as the matched-budget control of Section~\ref{subsec:rq2} established---but actively harmful.

The behaviour of the Uniform baseline is equally instructive: it does not collapse here, even though it caps every client at half width exactly as in the image experiments. The explanation lies in which parameters the cap leaves untrained. In the convolutional models, capping every client freezes the outer channels of every layer \emph{including the classifier}, so the frozen coordinates sit directly on the output path and corrupt predictions (Proposition~\ref{prop:frozen}). In the recurrent model, sub-model extraction slices hidden units within each gate block, but the character embedding and all rows of the output projection are shared in full by every client; the untrained coordinates are interior hidden units whose effect is to reduce effective capacity rather than to distort the output mapping. Parameter coverage thus matters most where it intersects the output layer, which sharpens rather than weakens the conclusion of Section~\ref{subsec:ablation_coverage}: what must be preserved is coverage of the parameters that determine predictions, not merely of the parameter count.

We note two limitations of this benchmark. Absolute accuracies (25--35\%) are well below published LEAF results, reflecting our deliberately small two-layer LSTM, the 200-round budget, and a single local epoch; the comparison across allocation policies is nevertheless internally consistent, since every policy trains the identical architecture under the identical protocol. Second, with twenty role-clients the federation is small, so the natural partition is heterogeneous but not large-scale.

\subsection{Implications for Update-Based Heterogeneity Estimation}
\label{subsec:implications}

The capacity confound of Section~\ref{subsec:rq2} suggests dividing the literature into two groups.
Methods that infer data properties from client updates while holding model size fixed are unaffected: clustered federated learning groups clients by the cosine similarity of equal-size gradients~\citep{sattler2020clustered}, IFCA assigns cluster identities by lowest local loss under candidate models~\citep{ghosh2020clustered}, FedGroup clusters by decomposed optimization-direction similarity~\citep{duan2021fedgroup}, and loss- or gradient-norm-based client selection~\citep{cho2022towards,marnissi2021client} likewise compares like with like.
In the equal-model regime these signals are legitimate data proxies.
They lose that meaning the moment sub-model width varies with device capacity: a quarter-width client mechanically produces smaller, differently oriented updates regardless of its data---precisely the coupling we measure ($r$ between $-0.72$ and $-0.84$ across all runs).
Methods that size or shape sub-models from training-derived signals---capability-driven pruning ratios~\citep{jiang2022fedmp}, data-driven channel importance~\citep{li2021hermes}, magnitude-based sub-model composition~\citep{wu2024fiarse}, and learned sparse ratios that explicitly superimpose capability and data effects~\citep{xue2025fedlps}---operate in the regime where the confound arises. We did not evaluate these methods, so we do not claim that their signals are confounded; we do argue that a data-property inference drawn from their updates may mix data with capacity and should be checked.
Our results suggest a concrete reporting practice for this setting: whenever client statistics are estimated from heterogeneous sub-model updates, validate them against a capacity-stratified analysis (as in Table~\ref{tab:rq2_correlations}) before attaching a data-heterogeneity interpretation.

\section{Conclusion}
\label{sec:conclusion}

This paper used Heterogeneity-Aware Adaptive Sub-model Federated Learning (HAS-FL) as a test case for studying a question that comes before any such design. Can client data heterogeneity be estimated from the sub-model updates a federated server observes, and does allocating capacity by that estimate help? The hypothesis we started from was that clients with divergent data need more capacity, while clients whose data matches the population can contribute through smaller sub-models. Our evidence does not support that hypothesis, and it shows why.

We report three findings. First, when validated against ground-truth label divergence on reproducible partitions, update-divergence estimates of heterogeneity are dominated by capacity rather than by data. The correlation with capacity falls between $-0.72$ and $-0.84$ on all six runs, and once capacity is controlled for, the remaining association with data heterogeneity is near zero or negative rather than positive. The cause is structural. Sub-models of different widths differ by an order of magnitude in the size and norm of their updates, which is far larger than any variation caused by data.

Second, capped allocation has a failure mode that is easy to miss and can be proven rather than merely observed. If every client is capped below full width, the uncovered coordinates never move from their initialization while still taking part in every prediction (Proposition~\ref{prop:frozen}). The damage ranges from gradual decay on CIFAR-10 to a fall to near-chance accuracy on the 62-class EMNIST task, and a coverage guarantee removes it. The Shakespeare results make the claim more precise. What must stay covered is the parameters on the output path, because a recurrent model whose embedding and output projection are shared in full can tolerate uncovered hidden units without collapsing.

Third, a matched-budget control settles the practical question. Random allocation to the same average budget matches the heterogeneity-aware policy to within seed-to-seed variation on both image benchmarks. On the naturally partitioned Shakespeare benchmark, the adaptive policy is the weakest of the three strategies tested while using the most capacity. Sub-model training itself remains useful, since half-width clients train and communicate at roughly a quarter of the cost, and this admits devices that could not otherwise take part. Even so, it gives up roughly eight to thirteen accuracy points against full-model training, and within the sub-model regime what protects accuracy is parameter coverage rather than allocation intelligence: across the average capacities we tested ($0.41$ to $0.56$), accuracy did not vary beyond seed-to-seed variation.

Several limitations apply. Our federations use twenty clients in simulation rather than physical devices, which is a scale typical of the sub-model literature we build on but small compared with production deployments. Absolute accuracies on Shakespeare are well below published LEAF results, because we use a deliberately small recurrent model and a short training budget. The comparisons remain valid within the study, since every policy trains an identical architecture under an identical protocol. The heterogeneity measure we study is one natural choice, and a negative result about it does not rule out estimators built from other signals. Any such estimator, however, faces the same difficulty of separating data effects from capacity effects while sub-model width varies with the device. Working out which data statistics can still be recovered in that setting, whether by estimating within capacity groups, by using probe batches of a fixed width, or by having clients report statistics directly, is the natural next step. The last option has recent precedents: clients can share mixture-model summaries of their local distributions in a single round~\citep{flores2026feddisc}, or jointly build a privacy-preserving summary of the distributed data through federated fuzzy clustering~\citep{ducange2026fedshap}. Signals of this kind are computed on the data rather than on sub-model updates, so they do not depend on sub-model width, although they introduce a disclosure cost that update-based estimates avoid.

The wider implication is a warning about a design pattern rather than about one algorithm. Inferring client properties from training signals is standard practice in federated learning, for client selection, weighting, and clustering, and it is sound when every client trains a model of the same size. Once sub-model training makes model size depend on device capacity, that inference changes meaning without any warning, and methods that set sub-model size from training signals are exposed to the same confound. We therefore recommend two practices. Any such method should check its estimates against an analysis that groups clients by capacity before treating them as a measure of data heterogeneity. Adaptive allocation schemes should be reported against a matched-budget random control, which in our experiments was enough to explain the benefits they report.

\backmatter

\bmhead{Author contributions}
Alireza Moayedikia: Conceptualisation, Methodology, Software, Investigation, Formal analysis, Writing -- original draft, Writing -- review \& editing. Alicia Troncoso Lora: Validation, Supervision, Writing -- review \& editing.

\bmhead{Funding}
No funding was received for this work.

\bmhead{Data availability}
All datasets used in this study (CIFAR-10, EMNIST and Shakespeare) are publicly available.

\bmhead{Code availability}
The complete implementation, configuration files, and scripts to reproduce all tables and figures are publicly available at \url{https://github.com/amoayedikia/hasfl}.

\section*{Declarations}

\bmhead{Conflict of interest}
The authors declare no conflict of interest.

\bmhead{Ethics approval}
Not applicable.

\bibliographystyle{sn-basic}
\bibliography{main}

\authorinfo

\end{document}

%% file: arxiv-style.tex
\documentclass[10pt,twoside]{article}
\usepackage[paperwidth=155mm,paperheight=235mm,
            top=21mm,bottom=20mm,inner=16.5mm,outer=16.5mm,
            headheight=12pt,headsep=7mm,footskip=8mm]{geometry}
\usepackage[T1]{fontenc}
\usepackage{mathptmx}
\usepackage[scaled=0.90]{helvet}
\usepackage{microtype}
\usepackage[round]{natbib}
\bibpunct{(}{)}{;}{a}{}{,}
\usepackage{titlesec}
\usepackage{fancyhdr}
\usepackage{lastpage}
\usepackage[labelfont={sf,bf},labelsep=quad,font=small,justification=raggedright,singlelinecheck=false]{caption}
\usepackage[colorlinks=true,citecolor=blue,linkcolor=blue,urlcolor=blue]{hyperref}
\usepackage{xstring}

\titleformat{\section}{\sffamily\bfseries\large\raggedright}{\thesection}{0.6em}{}
\titleformat{\subsection}{\sffamily\bfseries\normalsize\raggedright}{\thesubsection}{0.6em}{}
\titleformat{\subsubsection}{\sffamily\itshape\normalsize\raggedright}{\thesubsubsection}{0.6em}{}
\titlespacing*{\section}{0pt}{14pt plus 2pt}{6pt}
\titlespacing*{\subsection}{0pt}{12pt plus 2pt}{5pt}
\titlespacing*{\subsubsection}{0pt}{10pt plus 2pt}{4pt}
\fancypagestyle{firstpage}{\fancyhf{}\fancyhead[L]{\sffamily\footnotesize Preprint}}

\makeatletter
\newcommand\fnm[1]{#1}
\newcommand\sur[1]{#1}
\newcommand\orgdiv[1]{#1}
\newcommand\orgname[1]{#1}
\newcommand\orgaddress[1]{#1}
\newcommand\street[1]{#1}
\newcommand\city[1]{#1}
\newcommand\postcode[1]{#1}
\newcommand\state[1]{#1}
\newcommand\country[1]{#1}
\newcommand\email[1]{\g@addto@macro\@emaillist{\par\noindent\phantom{\textsuperscript{0}}\,\href{mailto:#1}{#1}}}
\def\@emaillist{}
\def\@authorlist{}
\def\@affillist{}
\newif\if@firstauthor \@firstauthortrue
\newcommand\@addauthor[2]{%
  \if@firstauthor\@firstauthorfalse\else\g@addto@macro\@authorlist{\ \textperiodcentered\ }\fi
  \g@addto@macro\@authorlist{#2\textsuperscript{#1}}}
\renewcommand\author{\@ifstar{\@authorstar}{\@authornostar}}
\newcommand\@authorstar[2][]{\@addauthor{#1}{#2}\g@addto@macro\@authorlist{\textsuperscript{\,\ensuremath{\boxtimes}}}}
\newcommand\@authornostar[2][]{\@addauthor{#1}{#2}}
\newcommand\affil{\@ifstar{\@affil}{\@affil}}
\newcommand\@affil[2][]{\g@addto@macro\@affillist{\par\smallskip\noindent\textsuperscript{#1}\,#2}}
\renewcommand\title[2][]{\gdef\@title{#2}}
\renewcommand\abstract[1]{\gdef\@abstract{#1}}
\newcommand\keywords[1]{\gdef\@keywords{#1}}
\renewcommand\maketitle{%
  \thispagestyle{firstpage}%
  \begingroup\raggedright
  \vspace*{6mm}
  {\sffamily\bfseries\fontsize{14}{17}\selectfont\@title\par}
  \vspace{5mm}
  {\sffamily\bfseries\normalsize\@authorlist\par}
  \vspace{9mm}
  {\sffamily\bfseries\normalsize Abstract\par}
  \endgroup
  \vspace{1mm}\noindent\@abstract\par
  \vspace{4mm}\noindent{\sffamily\bfseries Keywords}\enspace\@keywords\par
  \vspace{2mm}
  \let\thefootnote\relax
  \footnotetext{Extended author information available on the last page of the article}
  \addtocounter{footnote}{-1}%
  \vspace{4mm}}
\newcommand\backmatter{}
\newcommand\bmhead[1]{\par\medskip\noindent{\sffamily\bfseries #1}\enspace}

\newcommand\authorinfo{%
  \par\vspace{10mm}\noindent\rule{\textwidth}{0.4pt}\par\smallskip
  \noindent{\sffamily\bfseries\large Authors and Affiliations}\par\medskip
  {\small\noindent Alireza Moayedikia \textsuperscript{\,\ensuremath{\boxtimes}}\ \href{mailto:amoayedikia@swin.edu.au}{amoayedikia@swin.edu.au}\par
   \noindent Alicia Troncoso Lora\ \href{mailto:atrolor@upo.es}{atrolor@upo.es}\par
   \@affillist\par}}
\makeatother
\usepackage{xurl}